\documentclass[lettersize,journal]{IEEEtran}
\usepackage{amsmath,amsfonts}
\usepackage{algorithmic}
\usepackage{array}
\usepackage[caption=false,font=normalsize,labelfont=sf,textfont=sf]{subfig}
\usepackage{textcomp}
\usepackage{stfloats}
\usepackage{url}
\usepackage{verbatim}
\usepackage{graphicx}
\def\BibTeX{{\rm B\kern-.05em{\sc i\kern-.025em b}\kern-.08em
    T\kern-.1667em\lower.7ex\hbox{E}\kern-.125emX}}
\usepackage{balance}

\usepackage{amsmath} 
\usepackage{amsthm}
\usepackage{cases}

\usepackage{amsfonts}
\usepackage{amssymb}  
\usepackage{mathtools} 
\usepackage{mathrsfs}
\usepackage{balance}
\usepackage{nicefrac}
\usepackage{cite}
\usepackage{hyperref}
\hypersetup{
    colorlinks=true,
    linkcolor=blue,
    filecolor=magenta,      
    urlcolor=cyan,
}
\usepackage{graphicx}
\usepackage{textcomp}
\usepackage{color}
\usepackage{adjustbox}
\usepackage{caption}
\DeclareMathAlphabet{\mathpzc}{OT1}{pzc}{m}{it}

\newcommand{\R}{\mathbb{R}}

\theoremstyle{definition}

\newtheorem{definition}{Definition}

\newtheorem{lemma}{Lemma}
\newtheorem{theorem}{Theorem}
\newtheorem{remark}{Remark}

\usepackage{cleveref}

\usepackage[ruled,linesnumbered]{algorithm2e}

\newcommand{\norm}[1]{\left\lVert #1\right\rVert}

\DeclareMathOperator*{\argmin}{arg\,min}
\let\argmax\relax
\DeclareMathOperator*{\argmax}{arg\,max}

\newcommand{\doi}[1]{\href{http://dx.doi.org/#1}{\normalsize{\textsc{doi:}}~\nolinkurl{#1}}}
\newcommand{\arxiv}[1]{\href{http://arxiv.org/abs/#1}{\normalsize{\textsc{arxiv:}}~\nolinkurl{#1}}}
\usepackage{color}

\usepackage{comment}
\newcommand{\HRule}{\noindent\rule{\linewidth}{0.1mm}\newline}

\renewcommand{\phi}{\varphi}

\newcommand{\N}{\mathbb{N}}

\newcommand{\cbf}{C}

\newcommand{\cert}{\cbf}
\newcommand{\LfC}{L_{f} \cert}
\newcommand{\LgC}{L_{g} \cert}

\newcommand{\gpaugdomain}{\bar{\mathcal{X}}}
\newcommand{\gpvar}{x}
\newcommand{\gpvaraug}{\bar{x}}

\newcommand{\gpmuB}{\mu}
\newcommand{\gpsigmaB}{\sigma}
\newcommand{\gpGramB}{\Sigma}
\newcommand{\socc}{\widehat{L_g \cert}(x|\dataset)}
\newcommand{\soccsimple}{\widehat{L_g \cert}(x)}
\newcommand{\socconline}{\widehat{L_g \cert}(x|\datasetonline)}

\newcommand{\dataset}{\mathbb{D}_N}
\newcommand{\datasetonline}{\mathbb{D}_M}

\newcommand{\xquery}{x_{*}}

\newcommand{\uquery}{u_{*}}

\newcommand{\xuquery}{x_{*}, u_{*}}
\newcommand{\yquerycol}{\left[\!\!\begin{array}{c} 1 \\ \uquery \\\end{array}\!\!\right]}
\newcommand{\yqueryrow}{\left[1 \;\; \uquery^\top \right]}
\newcommand{\nqi}{n_{i}}
\newcommand{\nnqi}{n_{i}'}

\newcommand{\kadp}{\mathbf{k}}
\newcommand{\kadpu}{\mathbf{k}^{u}}
\newcommand{\Kadp}{K_{\dataset}}
\newcommand{\Kadponline}{K_{\datasetonline}}
\newcommand{\Diag}{\mathrm{Diag}}
\newcommand{\dataselectobjective}{F_{\datasetonline}}

\newcommand{\oneucol}{\left[\!\!\begin{array}{c} 1 \\ u \\\end{array}\!\!\right]}

\newcommand{\KqueryY}{K_{*U}}

\newcommand{\U}{U_{N}}

\newcommand{\revision}[1]{{\color{black} #1}}

\newcommand{\ours}{Ours}
\newcommand{\criticalconstraints}{system-critical constraints}
\newcommand{\criticalconstraint}{system-critical constraint}
\newcommand{\safetyfilter}{certifying filter}
\newcommand{\uref}{u_{\text{ref}}}

\definecolor{customblue}{rgb}{0.106, 0.588, 0.953}
\definecolor{custommagenta}{rgb}{0.937, 0.004, 0.584}
\definecolor{customorange}{rgb}{0.965, 0.529, 0.255}
\definecolor{customgreen}{rgb}{0.043, 0.694, 0.639}
\definecolor{customyellow}{rgb}{0.9373, 0.8941, 0.0235}
\definecolor{custompurple}{rgb}
{0.627, 0.125, 0.941}
\definecolor{customgrey}{rgb}
{0.4, 0.4, 0.4}
\newcommand{\textmagenta}[1]{{\color{custommagenta} #1}}
\newcommand{\textorange}[1]{{\color{customorange} #1}}
\newcommand{\textteal}[1]{{\color{customgreen} #1}}
\newcommand{\textblue}[1]{{\color{customblue} #1}}

\newcommand{\textpurple}[1]{{\color{custompurple} #1}}
\newcommand{\textgrey}[1]{{\color{customgrey} #1}}

\usepackage{siunitx}

\begin{document}

\title{Constraint-Guided Online Data Selection for Scalable Data-Driven Safety Filters in \\ Uncertain Robotic Systems}
\author{Jason J. Choi*, Fernando Castañeda*, Wonsuhk Jung*, Bike Zhang, Claire J. Tomlin, and Koushil Sreenath
\thanks{*The first three authors contributed equally to the work. J. J. Choi, F. Castañeda, B. Zhang, C. J. Tomlin, and K. Sreenath are with the University of California, Berkeley, CA, 94720, USA. W. Jung is with Georgia Institute of Technology, GA, 30332, USA.}}

\markboth{IEEE Transactions on Robotics,~Vol.~XX, No.~X, September~2023}%
{How to Use the IEEEtran \LaTeX \ Templates}

\maketitle

\begin{abstract}
As the use of autonomous robots expands in tasks that are complex and challenging to model, the demand for robust data-driven control methods that can certify safety and stability in uncertain conditions is increasing. However, the practical implementation of these methods often faces scalability issues due to the growing amount of data points with system complexity, and a significant reliance on high-quality training data. In response to these challenges, this study presents a scalable data-driven controller that efficiently identifies and infers from the most informative data points for implementing data-driven safety filters. Our approach is grounded in the integration of a model-based certificate function-based method and Gaussian Process (GP) regression, reinforced by a novel online data selection algorithm that reduces time complexity from quadratic to linear relative to dataset size. Empirical evidence, gathered from successful real-world cart-pole swing-up experiments and simulated locomotion of a five-link bipedal robot, demonstrates the efficacy of our approach. Our findings reveal that our efficient online data selection algorithm, which strategically selects key data points, enhances the practicality and efficiency of data-driven certifying filters in complex robotic systems, significantly mitigating scalability concerns inherent in nonparametric learning-based control methods.
\end{abstract}



\section{Introduction}
\label{sec:introduction}
\subsection{Motivation}
\IEEEPARstart{A}{s} autonomous robots become increasingly prevalent in our everyday lives, incidents such as fatal accidents involving self-driving cars have highlighted the importance of ensuring various system-critical constraints. Examples of these constraints include safety requirements for self-driving cars to prevent accidents, or stability conditions for legged robots to avoid falling over. Failure to meet these constraints can lead to catastrophic consequences.

Learning-based control methods have gained considerable attention in recent years due to their capacity to accomplish complex tasks by leveraging vast amounts of data.
However, a fundamental challenge in deploying these emerging methods for real-world robots is ensuring their adherence to the considered system-critical constraints.

One way of addressing this challenge is through \emph{a-posteriori verification}, i.e., analyzing these policies after they have been synthesized. A significant body of recent work on neural network verification follows this approach \cite{liu2021algorithms}. If the neural policy does not pass the certification test, the designer typically modifies the learning setup and reiterates this process until a satisfying policy is found.
This can be a time-consuming and computationally expensive process, although there are ongoing efforts in the community to overcome this challenge \cite{chang2019neural, tonkens2023patching}.

An alternative to the a-posteriori certification approach is using a model-based \textit{certifying filter}, supplemented by data to address the limitations imposed by imperfect models. In this scheme, a certifying filter is designed using available mathematical models to ensure compliance with system-critical constraints, and then the filter is enhanced by incorporating data from the actual system to address discrepancies arising from the imperfect model. This concept is often referred to as the \textit{data-driven safety filter} \cite{Wabersich2023}. 

\revision{The data-driven component of the filter learns how to address the effect of the discrepancy between the model and the actual system on the feasible control input that satisfies the desired constraints.} Many of these methods employ nonparametric learning techniques, such as Gaussian Process regression \cite{williams2006gaussian}, because it provides not only predictions but also a quantification of possible prediction errors of the learned model. This can be used to maintain the desired level of safety despite any errors the learned model may make.

However, there are two primary challenges associated with the data-driven certifying filter. First, nonparametric methods generally do not scale well with an increasing number of data points, which can significantly limit their application to more complex systems. As system complexity increases, characterizing the effect of the imperfect model requires more data, and controllers need to operate at higher frequencies to effectively manage rapidly changing system dynamics. Second, the success of a data-driven certifying filter in achieving its objective, like any other learning-based control approach, relies on the quality of the available training data \cite{lederer2021impact, castaneda2021pointwise}. If the information derived from the data about the real system is insufficient, the data-driven controller can fail to adhere to the system-critical constraints.

These two challenges are tightly coupled; when a large dataset is available, it is crucial to investigate which data points are most critical for meeting the certifying filter's objective to properly address the scalability challenge.
This raises the motivating question of the paper: by focusing on the most relevant data for system-critical constraints, can we extend the applicability of data-driven certifying filters to more complex and uncertain real-world robotic systems?

\subsection{Contributions}
\label{sec:contribution}

In this paper, we introduce an efficient approach for determining the most relevant data points for deploying data-driven certifying filters on real-world robotic systems. Our method expands on our earlier work \cite{GPCLFSOCP, castaneda2022probabilistic}, where we designed a data-driven filter that combines a model-based control method rooted in certificate functions, such as Control Barrier Functions (CBFs \cite{ames2017cbf}) and Control Lyapunov Functions (CLFs \cite{artstein}), with a Gaussian Process (GP) regression for the data-driven component that learns the effect of model uncertainty. \revision{The control input is filtered by solving a second-order cone program (SOCP), and is guaranteed to satisfy system-critical constraints with high probability when the SOCP is feasible. Such feasibility can be secured if the data is rich enough; however, with a large amount of data, the GP regression will slow down significantly due to its poor computational complexity. To resolve this issue, we delve into understanding which data points are critical for ensuring the feasibility of the SOCP filter and subsequently, for meeting the system-critical constraints.} Guided by this understanding, we develop an efficient online data selection algorithm for the filter. Each time the SOCP controller is executed, this algorithm selects only the most relevant data points to secure the SOCP's feasibility, which considerably enhances the time complexity of the GP-based safety filter.

Using the proposed approach, we showcase the applicability of Gaussian Process-based safety filters to high-dimensional and real robotic systems handling large datasets, overcoming the scalability constraints that previously limited the use of such filters to simple toy systems \cite{GPCLFSOCP, castaneda2021pointwise, taylor2020towards, castaneda2022probabilistic, dhiman2020control, brunke2022barrier, kumar2021probf, greeff2021learning}. We demonstrate the successful deployment of our method in a real cart-pole experiment to ensure the cart remains within its position limit and a 10-dimensional bipedal robot in simulation that attempts stable walking while subjected to model errors.

\subsection{Notations}

\footnotesize
\noindent$B$: binary correlation indicator matrix \eqref{eq:binary_matrix}

\noindent$\cert$: certificate function (Definition \ref{def:certificate})

\noindent$\dataset:=\{\gpvaraug_j,\bar{z}_j\}_{j=1}^{N}$: entire dataset for GP regression

\noindent$\datasetonline(\subset \dataset)$: online dataset

\noindent$f, g$: true plant vector fields (\eqref{eq:system})

\noindent$\tilde{f}, \tilde{g}$: nominal model vector fields (\eqref{eq:nominal-model})

\noindent$\dataselectobjective$: objective function of the data selection algorithm (\eqref{eq:data_select_objective})

\noindent$k_f, k_{g_1}\cdots , k_{g_m}$: individual kernels (Definition \ref{def:adpkernel})

\noindent$\kadp$: Affine Dot Product (ADP) kernel (Definition \ref{def:adpkernel})

\noindent$\kadpu$: ADP kernel that captures only the control vector field-relevant part (\eqref{eq:adpkernel_u})

\noindent$\Kadp, \Kadponline$: kernel matrix whose $(i, j)^{th}$ element is $k(\gpvaraug_i, \gpvaraug_j)$

\noindent$K_{*}\!:=\![k(\gpvaraug_*, \gpvaraug_1),\ \cdots \ ,k(\gpvaraug_*, \gpvaraug_N)]\in\R^N$

\noindent$\KqueryY$: \eqref{eq:KqueryY}

\noindent$\widehat{L_f \cert}(x|\mathbb{D}_{M, N})$, $\widehat{L_g \cert}(x|\mathbb{D}_{M, N})$: GP mean-based estimate of the Lie derivatives of $\cert$ (\eqref{eq:mean-based-estimate})

\noindent$m$: control input dimension

\noindent$M$: number of online data points

\noindent$n$: state dimension

\noindent$\nqi(x, u)$: kernel-based alignment measure (\eqref{eq:normalized_kernel_metric})

\noindent$N$: number of entire data points

\noindent$u$: control input

\noindent$\uref$: reference controller

\noindent$\mathcal{U}$: control input bound

\noindent$x$: state

\noindent$\gpvaraug:=(x, u)$: input for GP regression

\noindent$\mathcal{X}$: state domain

\noindent$\gpaugdomain = \mathcal{X} \times \R^{m}$: GP input domain

\noindent$\bar{z}_j$: noisy measurement of $\Delta$ at query $\gpvaraug_j$

\noindent$\mathbf{z}$: vector consisting of the dataset outputs, $\bar{z}_j$

\noindent$\beta$: constant in Assumption \ref{asump:well-calibrated}

\noindent$\gamma$: comparison function in Definition \ref{def:certificate}

\noindent$\delta$: probability level in Assumption \ref{asump:well-calibrated}

\noindent$\Delta$: model uncertainty term (\eqref{eq:mismatch_cbf})

\noindent$\epsilon$: constant in Theorem \ref{thm:main}

\noindent$\gpmuB (x,u|\mathbb{D}_{M, N})$: GP posterior mean

\noindent$\mathcal{M}(x|\dataset)$: \eqref{eq:mu_adp}

\noindent$\sigma^{2}(x,u|\mathbb{D}_{M, N})$: GP posterior variance

\noindent$\sigma_n^2$: measurement noise variance

\noindent$\Sigma(x|\dataset)$: Gram matrix of GP posterior variance (\eqref{eq:sigma_adp})

\noindent$\boldsymbol{\phi}(x, u)$: ADP kernel's feature vector
\normalsize
\section{Related Work}
\label{sec:related work}

Control Barrier Functions (CBFs, \cite{ames2017cbf}) and Control Lyapunov Functions (CLFs, \cite{artstein}) are model-based certificate functions that can be used to design policy filters to enforce safety and stability, respectively, of a controlled system. While initially conceived for systems with perfectly known dynamics, early results showed how to extend these filters to robust \cite{nguyen2017robust, jankovic2018robust, nguyen2021robust, choi2021robust} and adaptive \cite{nguyen2015l1adaptive, taylor2020adaptive, lopez2020robust} control settings to address the issue of imperfect models.

The integration of these certificate filters with data-driven methods has become increasingly popular for systems with uncertain dynamics. 
Several studies employ neural networks to learn the model mismatch terms \cite{taylor2019clflearning, taylor2020cbf, choi2020reinforcement}. Despite their practicality and effectiveness, verifying the accuracy of neural network predictions can be challenging. Alternative approaches, upon which our work builds, use nonparametric regression techniques \cite{GPCLFSOCP, castaneda2021pointwise, castaneda2022probabilistic, taylor2020towards, dhiman2020control, brunke2022barrier, greeff2021learning, kumar2021probf, kazemian2022random}.
Most notably, Gaussian Process (GP) regression models provide a probabilistic assurance of prediction quality under certain assumptions \cite{gpucb, lederer2019uniform}. 

The GP research community has a rich history in developing methods to improve the computational complexity of GP inference, commonly referred to as Sparse GP regression \cite{snelson2005sparse, quinonero2005unifying, liu2020gaussian}. The work in \cite{kazemian2022random} uses one of these methods (random features approximation) to speed up GP inference for data-driven safety filters.
Additionally, existing approaches that quantify the importance of data for system identification mostly focus on optimizing information-theoretic metrics, such as the information gain, when developing exploration strategies \cite{krause2008near, pukelsheim2006optimal, kim2018emi, alpcan2015information, contal2014gaussian, koller2018learning}.
However, these general-purpose methods lack awareness of any control objective. \revision{Instead of aiming for a good global approximation quality of the GP inference, the GP inference in our method is approximated by a small set of data points selected at each state and is tailored for the purpose of determining a safe control input for each state.} Obtaining the best subset of data, in fact, constitutes a combinatorial optimization problem that would be more computationally demanding than performing exact GP inference. For this reason, we instead present a control-informed efficient approximate data selection method that effectively serves to reduce the inference time of data-driven safety filters. This enables the deployment of these filters on real robotic systems.

The authors of \cite{lederer2020training, lederer2021impact} propose a method to evaluate the importance of data for maintaining the stability of data-driven closed-loop systems. As such, they study the connection between data and the performance of a particular given policy. Additionally, they introduce a greedy data selection strategy for GP inference based on an importance measure they propose. However, these selection strategies are still too computationally expensive to run online. \revision{Our work instead tackles the problem of robust control design by examining which data is most relevant to the control input that can achieve the desired certification property.} Furthermore, our approach characterizes this relevance in the control input space, emphasizing the richness of each data point for the specific certification objective, rather than relying on data density measures. This is a similar objective to the one of \cite{capone2020data}, where an algorithm to select the most useful data points for successfully performing multiple control tasks is presented. However, this method suffers from the scalability issue that we would like to tackle in this work.

\section{Certifying Filters for Uncertain Systems}
\label{sec:background}
\subsection{Uncertain Dynamics and Certifying Filter}

In this work, we examine a system described by
\begin{equation}\label{eq:system}
\dot{x} = f(x) + g(x)u,
\end{equation}
which is control-affine. Here, $x \in \mathcal{X} \subset \R^n$ represents the state, and $u \in \mathbb{R}^m$ denotes the control input. This form is suitable for representing various robotic systems, including those with Lagrangian dynamics. We assume that both $f$ and $g$ are locally Lipschitz continuous, and without loss of generality, we consider $f(0) = 0$ so that $x=0$ is an equilibrium, and assume $0 \in \mathcal{X}$. Throughout the paper, we will refer to the system described in \eqref{eq:system} as the \textit{true plant}.

This paper addresses the challenge of ensuring critical system constraints for the true plant \eqref{eq:system}, such as safety and stability, when its dynamics $f$ and $g$ are unknown, while trying to accomplish a desired task. We assume that a controller for the desired task has already been designed and is provided as a \textit{reference controller} $\uref: \mathcal{X} \to \R^m$. In the absence of the reference controller, we can consider $\uref(x)\equiv0$. This controller is often unaware of the system's constraints that are vital for preventing catastrophic failure, which we refer to as \textit{\criticalconstraints}. Common examples of these constraints include safety constraints, which can be expressed as constraints in the system's state space, and stability constraints that maintain the system's stability around a desired equilibrium point.

We aim to design a \textit{\safetyfilter} that operates between the reference controller $\uref$ and the true plant, ensuring the control applied to the true plant is filtered to satisfy the relevant \criticalconstraint. When the reference controller $\uref$ adheres to the constraint, the \safetyfilter~simply passes $\uref(x)$ to the true plant. However, if $\uref$ violates the constraint, the filter minimally overrides it with a safe control signal to prevent system failure. This filtering structure is known by various names, most notably as a \textit{safety filter} \cite{Wabersich2023, hsu2023safety}. \revision{Since this structure decouples the design process for safety assurance (i.e. the design of the filter) from the design procedure for achieving performance (i.e. the design of $\uref$), it reduces the complexity of the control system design. As a result, it has been demonstrated to be an effective control architecture for numerous real-world applications \cite{hobbs2023runtime}.}

For this purpose, we assume access to an approximate \textit{nominal model} of the true plant's dynamics, represented by $\tilde{f}: \mathcal{X} \to \R^n$ and $\tilde{g}: \mathcal{X} \to \R^{n\times m}$:
\begin{equation}\label{eq:nominal-model}
\dot{x} = \tilde{f}(x) + \tilde{g}(x)u.
\end{equation}
Consequently, the true plant dynamics in \eqref{eq:system} are uncertain, and only a nominal model \eqref{eq:nominal-model} is available. This nominal model serves as the starting point for the design steps of the \safetyfilter, which will be discussed subsequently, and represents the designer's best estimate of the true plant. 

\subsection{Certificate Function-based Design}

A vital step in designing the \safetyfilter~involves utilizing the concept of \textit{certificate functions} \cite{dawson2023safe}, which are also known by various names, such as safety index in \cite{liu2014control} or energy function in \cite{wei2019safe}. Informally, a certificate function is a scalar function of the state, and its value and gradient can be used to establish a sufficient condition for a control input $u$ to satisfy the desired \criticalconstraint. This condition can then be employed as a \textit{certifying constraint} in the \safetyfilter~for the control input. If $\uref(x)$ fails to meet the constraint, it is overridden with an appropriate control input $u$ that satisfies the constraint. 

In this paper, we focus on Control Barrier Functions (CBFs) \cite{ames2017cbf} and Control Lyapunov Functions (CLFs) \cite{artstein} as specific examples of certificate functions, since they are the most prevalent choices for ensuring the satisfaction of safety and stability constraints in a system, respectively \cite{dawson2023safe}. We put forth a definition of a certificate function that unifies both CBFs and CLFs under a single definition, as similarly proposed in \cite{taylor2020towards}:

\begin{definition}
\label{def:certificate}
A function $\cert :\mathcal{X} \rightarrow \R$ is a \textit{certificate function} for the true plant \eqref{eq:system} with an extended class $\mathcal{K}_\infty$ function $\gamma:\R \rightarrow \R$ (called comparison function) if 
\begin{enumerate}
    \item for all $x\in \mathcal{X}$, there exists $u\in \R^m$ such that 
\begin{equation}
    \dot{\cert}(x, u) + \gamma(\cert(x)) \ge 0,
    \label{eq:certificate_condition}
\end{equation}
where $\dot{\cert}(x, u)$ is the \revision{time derivative of $\cert$} (with respect to the true plant \eqref{eq:system}), that is,
\begin{equation}\dot{\cert}(x, u) = 
    \underbrace{\nabla \cert(x) \cdot f(x)}_{\LfC(x)}+\underbrace{\nabla \cert(x) \cdot g(x)}_{\LgC(x)} u,
\end{equation}
    \item and if $u(t)$ satisfying \eqref{eq:certificate_condition} for all $t\ge0$ is a sufficient condition for $x(t)$ satisfying a desired \criticalconstraint~for all $t\ge0$.
\end{enumerate}
\end{definition}

The \criticalconstraint~for the CBF is that the trajectory stays inside the zero-superlevel set of $\cert$ indefinitely, i.e., $x(t) \in \mathcal{C} := \{x \in X\;|\;\cert(x)\ge 0\}$ for all $t\ge 0$ \cite{ames2017cbf}. The \criticalconstraint~for the CLF is that the trajectory is asymptotically stable to the equilibrium $x=0$ \cite{ames2014rapidly}. Both CBFs and CLFs satisfy the aforementioned definition of certificate functions. Note that to align with the inequality form in \eqref{eq:certificate_condition}, we need to negate the CLF. This adjustment ensures that both CBFs and CLFs can be used within the same framework to satisfy the desired \criticalconstraints.

Given a reference controller $u_{\text{ref}}: \mathcal{X} \to \R^m$, the condition in \eqref{eq:certificate_condition} can be used to formulate a minimally-invasive \safetyfilter \cite{ames2017cbf}:

{\small
\HRule
\noindent \textbf{Certificate Function-based Quadratic Program (CF-QP)}:
\begin{subequations}
\label{eq:cbf-qp-all}
\begin{align}
u^{*}(x) & = & & \underset{u\in \R^{m}}{\argmin}  \quad \norm{u-u_{\text{ref}}(x)}_2^2 \label{eq:cbf-qp-cost}\\
& \text{s.t.} & & \LfC(x) + \LgC(x)u + \gamma (\cert(x)) \geq 0. \label{eq:cbf-constraint} \vspace{-.5em}
\end{align}
\end{subequations}
\hrule
\normalsize
\vspace{2mm}}

\noindent The constraint \eqref{eq:cbf-constraint} is affine in $u$, which means that the optimization problem is a quadratic program (QP). This problem is solved pointwise in time to obtain a filtered control law $u^* : \mathcal{X} \to \R^m$ that only deviates from the reference controller $\uref$ when the condition \eqref{eq:cbf-constraint} is violated. We will refer to \eqref{eq:cbf-constraint} as the \textit{true certifying constraint} and \eqref{eq:cbf-qp-all} as the \textit{oracle CF-QP}. When specifically using CBFs or CLFs in place of $\cert$, we may refer to \eqref{eq:cbf-qp-all} as the oracle CBF-QP and CLF-QP, respectively.

The primary assumption we make in this paper is that we have access to the certificate function $\cert$ that is valid for the true plant. The nominal model \eqref{eq:nominal-model} can be used to design such certificate functions, which is discussed further in Remark \ref{rmk:designing_cfs}. This assumption ensures that a control policy exists to keep the true plant \eqref{eq:system} in compliance with the \criticalconstraint. However, even when a valid certificate function is available, obtaining such a control policy is not straightforward with no direct access to $f$ and $g$ in the true certifying constraint \eqref{eq:cbf-constraint}. This is the main problem this work aims to solve.

\revision{\begin{remark}
\label{rmk:designing_cfs} \textit{(Finding valid certificate functions)} Designing certificate functions for uncertain systems is far from trivial and is, in fact, an active area of research \cite{dawson2022safe, qin2022sablas, jagtap2020control, lindemann2024learning, jin2020neural, wei2023presistently}. Our contribution runs parallel to this line of research, and in fact, our work complements these efforts, as only when the design of the certificate function and the design of the certifying filter are combined, can the \safetyfilter~for uncertain systems be effectively implemented. In our work, we employ the nominal model to find CBFs and CLFs to be used as certificate functions. Thus, this procedure implicitly assumes that the nominal model is sufficiently accurate in its approximation of the true plant to enable the identification of a valid CBF or CLF. This assumption is also present in prior works that most closely align with our research \cite{taylor2020cbf,taylor2020towards, choi2020reinforcement, dhiman2020control, greeff2021learning}. This approach is considered reasonable for feedback linearizable systems with known relative degree, owing to the inherent robustness properties of CBFs and CLFs \cite{xu2015robustness, kolathaya2018input}. Indeed, the practice of using first-principle nominal models for designing CBFs is widely studied and adopted for numerous complex robotics systems \cite{wu2015safety, khazoom2022humanoid, molnar2023safety, cohen2024safety}.
\end{remark}
} 

The first naive approach we can take is to use the nominal model and replace $\LfC(x)$ and $\LgC(x)$ with $L_{\tilde{f}}\cert(x)$ and $L_{\tilde{g}}\cert(x)$ respectively, the Lie derivatives of $\cert$ with respect to the nominal model. We call this a \textit{nominal model-based CF-QP}. However, due to the mismatch between the true plant dynamics and the nominal model, the nominal model-based CF-QP does not provide any guarantee that the \criticalconstraint~will be met under the filtered control input. To examine this, the true certifying constraint in \eqref{eq:cbf-constraint} is expressed using the nominal model as follows:
\begin{equation}
    \label{eq:cbf-constraint-uncertainty}
    \underbrace{L_{\tilde{f}}\cert(x) + L_{\tilde{g}}\cert(x)u}_{\widetilde{\dot{\cert}}(x, u)} + \Delta(x, u) + \gamma(\cert(x)) \ge 0,
\end{equation}
where $\widetilde{\dot{\cert}}(x, u)$ is the estimated time derivative of $\cert$ based on the nominal model, and the \textit{model uncertainty term} $\Delta$ \cite{nguyen2021robust, choi2020reinforcement}, is defined for each $x \in \mathcal{X}$, $u \in \R^m$ as
\begin{align}
    \Delta(x, u) := & (\LfC \!-\!L_{\tilde{f}}\cert)(x) + (L_g \cert\!-\!L_{\tilde{g}}\cert)(x)u  \nonumber \\
    = & \left[L_{\Delta f}\cert(x)\quad L_{\Delta g}\cert(x) \right] \left[\!\!\begin{array}{c} 1 \\ u \\\end{array}\!\!\right]. \label{eq:mismatch_cbf}
\end{align}
Note that like the original constraint \eqref{eq:cbf-constraint}, $\Delta$ is also affine in the control input $u$.

In the next section, we introduce a method developed in our prior work to estimate $\Delta$ from data collected from the true plant \cite{GPCLFSOCP, castaneda2021pointwise}. Note that the data for $\Delta$ can be gathered from state trajectories without needing access to $f$ and $g$. This can be achieved by evaluating $\dot{\cert}(x, u)$ along the trajectories using numerical differentiation and subtracting $\widetilde{\dot{\cert}}(x, u)$. By employing the estimate of $\Delta$ learned from the data, we can design a data-driven \safetyfilter~that offers a high probability of satisfying \eqref{eq:cbf-constraint}.

\section{Data-driven Certifying Filters}

The data-driven certifying filters we introduce in this section employ Gaussian Process (GP) regression to learn the estimate of $\Delta$ from data. We first provide a brief background on GP regression.

\subsection{Gaussian Process Regression}
\revision{A Gaussian Process is a random process where any finite collection of samples has a joint Gaussian distribution. It is characterized by the mean function $q: \gpaugdomain \rightarrow \R$, which determines the mean value of the sample, and the covariance (or kernel) function $k: \gpaugdomain \times \gpaugdomain \rightarrow \R$, which determines the covariance of the joint Gaussian distribution of the samples, where $\gpaugdomain$ represents the input domain of the process. A sample process of the GP is denoted as $h \sim \mathcal{GP}(q, k)$.}

\revision{Given the measurement data of an unknown function $h:\gpaugdomain \rightarrow \R$, a regression problem can be formulated by taking a Bayesian approach for inferring a posterior distribution of $h$, by assuming that $h$ is a sample from a GP.} This implies that the prior distribution of $h(\gpvaraug_*)$, where $\gpvaraug_* \in \gpaugdomain$ is an unseen query point, is given by $\mathcal{N}(q(\gpvaraug_*), k(\gpvaraug_*, \gpvaraug_*))$. For our application, $\gpaugdomain = \mathcal{X} \times \R^{m}$, where $\R^{m}$ represents the control input space, $\gpvaraug_* = (\gpvar_*, u_*)$, and the unknown function we aim to regress is $\Delta$ defined in \eqref{eq:mismatch_cbf}. We assume the mean function $q \equiv 0$ since the prior information, which is based on the nominal model, is already captured in the term $\widetilde{\dot{\cert}}(x, u)$ in \eqref{eq:cbf-constraint-uncertainty}.

With the dataset of noisy measurements of $\Delta$, denoted by $\bar{z}_j$ at query $\gpvaraug_j$, given as $\dataset:=\{\gpvaraug_j,\bar{z}_j\}_{j=1}^{N} = \{(\gpvar_j,u_j),\ \Delta(\gpvar_j,u_j)+\epsilon_j\}_{j=1}^{N}$, a prediction of $\Delta$ at $\gpvaraug_*$ is derived from the joint distribution of $[\Delta(\gpvaraug_1), \cdots, \Delta(\gpvaraug_N), \Delta(\gpvaraug_{*})]^{\top}$ conditioned on the dataset $\dataset$. \revision{Here, the distribution of the measurement noise, $\epsilon_j$, has zero mean and its variance is $\sigma_{n}^2$, with $\sigma_{n}>0$.} The conditional distribution at the query point $\gpvaraug_*$ is called the \textit{GP posterior}, whose mean and variance of the prediction of $\Delta(\gpvaraug_*)$ are expressed as
\begin{equation}
\label{eq:gpposteriormu}
        \mu(\gpvaraug_*|\dataset) = \mathbf{z}^{\top} (\Kadp + \sigma_n^2 I )^{-1} K_{*}^{T},
\end{equation}
\begin{equation}
\label{eq:gpposteriorsigma}
    \sigma^{2}(\gpvaraug_*|\dataset) = k\left(\gpvar_{*}, \gpvar_{*}\right)-K_{*}  (\Kadp + \sigma_n^2 I )^{-1} K_{*}^{T},
\end{equation}
where $\Kadp\in\R^{N\times N}$ is the kernel matrix, whose $(i, j)^{th}$ element is $k(\gpvaraug_i, \gpvaraug_j)$, $K_{*}\!:=\![k(\gpvaraug_*, \gpvaraug_1),\ \cdots \ ,k(\gpvaraug_*, \gpvaraug_N)]\in\R^N$, and $\mathbf{z}\in \R^N$ is the vector consisting of the dataset outputs, $\bar{z}_j$.

The kernel value, $k(\gpvaraug, \gpvaraug')$, quantifies the correlation between two query points $\gpvaraug$ and $\gpvaraug'$. A higher kernel value indicates a stronger correlation between the points, suggesting that the corresponding values of $\Delta(\gpvaraug)$ and $\Delta(\gpvaraug')$ are more likely to be similar to each other. Thus, the choice of kernel $k$ determines properties of the target function like its smoothness or Lipschitz constant \cite{lederer2019uniform}. For example, the square exponential kernel, which is among the most popular choices for kernels in the GP regression literature \cite{williams2006gaussian}, attributes a higher correlation to points that are closer together in the input space, and the resulting samples of the GP are infinitely differentiable functions. The kernel can also capture prior structural knowledge of the target function \cite{duvenaud2014automatic}.
In our case, we mainly want to exploit the fact that the target function $\Delta$ from \eqref{eq:mismatch_cbf} is control-affine. For this, we use the Affine Dot Product compound kernel presented in \cite{GPCLFSOCP}.

\begin{definition} \label{def:adpkernel}
\emph{Affine Dot Product Compound Kernel \cite{GPCLFSOCP}: }
Define $\kadp:\gpaugdomain \times \gpaugdomain \rightarrow \R$ given by
\vspace{-1em}

\small
\begin{align}
    & \kadp \! \left((x, u), (x', u')\right)  \nonumber \\
    & :=  [1 \; u^{\top}] \Diag(k_f(x, x'), k_{g_1}(x, x') \cdots\!, k_{g_{m}}(x, x')) \!\left[\!\!\begin{array}{c} 1 \\ u' \\\end{array}\!\!\right],
\label{eq:adpkernel}
\end{align}
\normalsize

\noindent where $\Diag(\cdot)$ indicates the diagonal matrix whose diagonal terms consist of the entities in the paranthesis, as the \emph{Affine Dot Product} (ADP) compound kernel of $(m\!+\!1)$ individual kernels $k_f, k_{g_1},\cdots,$ $k_{g_{m}}:\mathcal{X}\times \mathcal{X} \rightarrow \R$.
\end{definition}

When $L_{\Delta f}\cert(x)$ and each element of $L_{\Delta g}\cert(x)$ in \eqref{eq:mismatch_cbf} is a sample from a GP defined by the individual kernels $k_f, k_{g_1}\cdots ,$ $k_{g_{m}}$, respectively, the model uncertainty term $\Delta (x, u)$ is a sample from a GP defined by the ADP kernel $\kadp$. Thus, using the ADP compound kernel, from \eqref{eq:gpposteriormu} and \eqref{eq:gpposteriorsigma}, the GP posterior at a query point $(\xuquery)$ is given as
\begin{equation}
\label{eq:mu_adp}
    \mu(\xuquery|\dataset) = \underbrace{\mathbf{z}^{\top} (\Kadp + \sigma_n^2 I )^{-1} \KqueryY^{\top}}_{=:\ \mathcal{M}(\xquery|\dataset)} \yquerycol,
\end{equation}
\vspace{-5pt}
{\small
\begin{equation}
\label{eq:sigma_adp}
    \sigma^{2}(\xuquery|\dataset) \!= \!\yqueryrow\!\underbrace{\left(K_{**} \!-\!\KqueryY\!(\Kadp + \sigma_n^2 I )^{-1} \KqueryY^{T}\! \right)}_{=:\ \Sigma(\xquery|\dataset)}\! \yquerycol,
\end{equation}
\vspace{-6pt}}

\noindent where $K_{**} \!=\! \!\Diag \left(k_f(\xquery, \xquery), \cdots, k_{g_m}(\xquery , \xquery)\right) \!\in\! \R^{(m+1)\times(m+1)}$, and $\KqueryY\in\R^{(m+1)\times N}$ is given by

{\small
\begin{equation}
    \KqueryY\!:=\!\begin{bmatrix} k_f(\xquery, x_1)\; \cdots\; k_f(\xquery, x_N) \\ k_{g_{1}}(\xquery, x_1)\; \cdots\; k_{g_{1}}(\xquery, x_N) \\ \vdots \\k_{g_{m}}(\xquery, x_1)\; \cdots\; k_{g_{m}}(\xquery, x_N)
    \end{bmatrix}\!\circ\! \begin{bmatrix} \mathbf{1}^{1\times N} \\ \U
    \end{bmatrix},
    \label{eq:KqueryY}
\end{equation}}

\noindent where $\circ$ indicates the element-wise product, and $\U :=[u_1\; \cdots \; u_N]\in \R^{m \times N}$. Note that $\Sigma(\xquery|\dataset)$ is positive definite when the individual kernels $k_f, k_{g_1}\cdots ,$ $k_{g_{m}}$ are positive definite kernels and $\sigma_n > 0$ \cite{GPCLFSOCP}. 

\revision{A significant feature of the GP regression is that it generates predictions of the target function value as a probability distribution rather than deterministically. This allows for obtaining probabilistic \textit{uniform error bounds} of the GP prediction using $\mu(\xuquery|\dataset)$ and $\sigma(\xuquery|\dataset)$, defined as below:
\begin{definition} 
\label{asump:well-calibrated}
\emph{Uniform Error Bound \cite{lederer2019uniform}:} For a given $\delta \in (0, 1)$, if there exists a constant $\beta > 0$ such that 
{\small\begin{multline}
\label{eq:well_calibrated_assumption}
    \mathbb{P}\bigg\{\ \bigg| \mu(\xuquery|\dataset) - \Delta(\xuquery) \bigg| \leq \beta \sigma(\xuquery|\dataset)\bigg\}\geq 1-\delta
\vspace{-3pt}
\end{multline}}
\noindent \!\!\!holds \textit{for all} $\gpvar_{*} \in \mathcal{X},\ u_* \in \R^m$, we define the \textit{uniform error bound} of the GP model (for the target function $\Delta(\cdot)$) as $\begin{bmatrix}\mu(\cdot|\dataset) - \beta \sigma(\cdot|\dataset), & \mu(\cdot|\dataset) + \beta \sigma(\cdot|\dataset)\end{bmatrix}$.
\end{definition}

Numerous existing works have determined various sets of conditions under which such uniform error bounds exist \cite{gpucb, chowdhury2017newgpucb, Fisac2018, lederer2019uniform, fiedler2021practical, capone2022gaussian}, along with the corresponding values of $\beta$. Since the conditions vary across different works---for instance, some assume the target function is sourced from a Reproducing Kernel Hilbert Space \cite{gpucb, chowdhury2017newgpucb, fiedler2021practical}, while others assume it is a sampled random process from the GP $\mathcal{GP}(0, k)$ \cite{Fisac2018, lederer2019uniform}---we refer readers to the aforementioned literature for more details. Here, we assume that the designer has chosen and verified an appropriate set of assumptions and the corresponding $\beta$ that serves as a valid uniform error bound of the GP model. 

It is important to note that the subsequent technical development in this paper is not bound to any specific methods used to characterize the uniform error bound, as long as it is a valid bound. Additionally, it is worth noting that the uniform error bound can be characterized even when the hyperparameters of the kernel function are not accurately known, allowing our approach to have robustness against certain degrees of model misspecifications \cite{fiedler2021practical, capone2022gaussian}.
} 

\subsection{Second-order Cone Program-based Certifying Filters}

With the bound provided in \eqref{eq:well_calibrated_assumption}, we can now present a data-driven certifying filter that offers a high probability guarantee of satisfying \eqref{eq:cbf-constraint} based on the learned GP model of $\Delta$. By employing the lower bound of $\Delta(x, u)$, we construct a \textit{certifying chance constraint} that can be evaluated without explicit knowledge of the true plant's dynamics:
\begin{equation}
    L_{\tilde{f}}\cert(x) + L_{\tilde{g}}\cert(x) u+\gpmuB(x,u|\dataset)\!-\!\beta \gpsigmaB(x,u|\dataset)\!+\!\gamma (\cert(x))\!\geq\!0. \label{eq:socp-cbf-constraint}
\end{equation}
If the constraint \eqref{eq:socp-cbf-constraint} is satisfied, from Definition \ref{asump:well-calibrated}, we have a guarantee that the true certifying constraint in \eqref{eq:cbf-constraint} is satisfied with a probability of at least $1 - \delta$. 

The affine structure of the mean expression in \eqref{eq:mu_adp} gives
\begin{equation*}
    \gpmuB (x,u|\mathbb{D}_N)\!=\!\mathcal{M}(x|\mathbb{D}_N)\!\oneucol = \left[\widehat{L_{\Delta f}\cert}(x)\quad \widehat{L_{\Delta g}\cert}(x) \right]\!\oneucol, \label{eq:mu_feas}
\end{equation*}
where
\begin{align*}
    \widehat{L_{\Delta f}\cert}(x):=\!\mathcal{M}(x|\dataset)_{[1]}, \quad
    \widehat{L_{\Delta g}\cert}(x):=\!\mathcal{M}(x|\dataset)_{[2:(m+1)]}.
\end{align*}
Next, we define 
\begin{align}
    \widehat{L_f \cert}(x|\dataset) &:= L_{\tilde{f}}\cert(x) + \widehat{L_{\Delta f}\cert}(x) \in \R, \nonumber \\
    \socc &:= L_{\tilde{g}}\cert(x) + \widehat{L_{\Delta g}\cert}(x) \in \R^{1\times m}.
    \label{eq:mean-based-estimate}
\end{align}
Using these expressions, \eqref{eq:socp-cbf-constraint} can be represented as
\begin{equation}
\label{eq:socp-cbf-constraint2} 
    \beta \gpsigmaB(x,u|\dataset)  \le \left[\widehat{\LfC}(x|\dataset)\!+\!\gamma(\cert(x)) \quad \widehat{\LgC}(x|\dataset)\right]\!\oneucol.
\end{equation}
From the quadratic structure of the variance expression in \eqref{eq:sigma_adp} and $\Sigma(x|\dataset)$ being positive definite, we can conclude that \eqref{eq:socp-cbf-constraint2} is a second-order cone constraint \cite{GPCLFSOCP}. This constraint is then incorporated into a chance-constrained reformulation of the CF-QP:

{\small
\vspace{2mm}
\hrule
\vspace{2mm}
\noindent \textbf{GP-CF-SOCP} \cite{GPCLFSOCP, castaneda2021pointwise}:
\begin{align}
\label{eq:gp-cbf-socp}
&u^{*}(x) =   \underset{u\in U}
{\argmin}\norm{u-u_{\text{ref}}(x)}_2^2 \ \  \text{s.t.}
 \\
& L_{\tilde{f}}\cert(x) + L_{\tilde{g}}\cert(x) u \!+\!\gpmuB(x,u|\dataset)\!-\!\beta \gpsigmaB(x,u|\dataset)\!+\!\gamma (\cert(x)) \geq 0, \nonumber
\end{align}
\hrule
\vspace{2mm}}
\noindent wherein by leveraging the control-affine structure in the GP regression, we obtain a convex second-order cone program (SOCP) that can be solved efficiently at high-frequency rates using modern solvers. When specifically using CBFs or CLFs in place of $\cert$, we refer to \eqref{eq:gp-cbf-socp} as GP-CBF-SOCP and GP-CLF-SOCP, respectively. 

\revision{The guarantee that the true certifying constraint will be satisfied with high probability exists when the SOCP filter in \eqref{eq:gp-cbf-socp} is feasible \cite{GPCLFSOCP, castaneda2021pointwise, castaneda2022probabilistic}. An adequate amount of data can be collected, so that the feasibility of the SOCP filter is secured across the state domain where the controller will be deployed. Once the SOCP is feasible, solving the optimization is not a challenge since it is a convex program. However, the main challenge in executing the filter online lies in the computationally demanding evaluation of $\gpsigmaB$ when the size of the dataset is large. Such a large dataset might be required to secure the feasibility of the filter, especially for higher dimensional systems or highly uncertain systems.} 

The time complexity of the matrix inverse in \eqref{eq:sigma_adp}, $(\Kadp + \sigma_n^2 I )^{-1}$, is $\mathcal{O}(N^3)$, while the remaining matrix multiplication involved in evaluating $\Sigma(\xquery|\dataset)$ has a time complexity of $\mathcal{O}(N^2)$. Although the matrix inversion can be performed offline, when the dataset is large, the $\mathcal{O}(N^2)$ complexity still remains challenging. This issue primarily motivates the constraint-guided online data selection algorithm proposed in the next section.

\revision{\begin{remark} \textit{(Infeasibility of the SOCP filter)}
\eqref{eq:gp-cbf-socp} can become infeasible for two reasons. Firstly, there might not be any $u$ such that the prediction uncertainty in the left-hand side of \eqref{eq:socp-cbf-constraint2} is adequately small compared to its right-hand side. This suggests that the dataset for the GP regression is insufficient to characterize the model uncertainty term $\Delta$ with high confidence. In this case, it may be necessary to collect more data to reduce the prediction uncertainty and ensure the feasibility of \eqref{eq:socp-cbf-constraint2}. A detailed analysis of the conditions under which \eqref{eq:gp-cbf-socp} is feasible is conducted in \cite{castaneda2021pointwise}. Secondly, most robotic systems have bounded control input limits, either due to their physical actuation limits or safety concerns. This input bound, represented as $\mathcal{U} \subset \R^m$, further constrains the feasible set of \eqref{eq:gp-cbf-socp} and may render it infeasible. During the deployment of the SOCP filter on real-world systems, it is often impossible to perfectly eliminate infeasibility. However, an effective strategy to address cases when infeasibility occurs is to use a backup control input computed by the following second-order cone program, which is always feasible:
\begin{equation}
\label{eq:backup_controller}
u^{*}(x) = \underset{u\in \mathcal{U}}{\argmin} \left(\beta \gpsigmaB(x,u|\dataset) - \widehat{\LgC}(x|\dataset)u\right).
\end{equation}
This selects a control input within the input bound that minimizes the violation of the constraint \eqref{eq:socp-cbf-constraint2}.
\end{remark}
} 

\subsection{Running Example: 2D Polynomial System (Polysys)}
\label{subsec:running_example_intro}

We now introduce a simple running example, referred to as Polysys, which is utilized throughout the paper. It is important to note that this low-dimensional toy example is not intended to showcase the computational advantage of our method or to represent a realistic setting. Instead, its purpose is to provide a walk-through of the inner workings of our approach for the readers. To achieve this, we have access to the true plant dynamics, allowing us to compare our method with the ideal oracle certifying filter. 

The dynamics of the system, whose vector fields are polynomial functions of the state $x$, are given by:
\begin{equation}
\label{eq:polysys_dynamics}
    \dot x = 
        \begin{bmatrix} f_{1}^{T}v \\ f_{2}^{T}v\end{bmatrix} + 
        \begin{bmatrix} 1+g_{11}^{T}v & g_{12}^{T}v  \\ g_{21}^{T}v & 1+g_{22}^{T}v\end{bmatrix} u,
\end{equation}
where $x = [x_1 \; x_2]^\top$ is the state, $u=[u_1 \; u_2]^\top$ is the control input, $v = [x_1 \; x_2 \; x_1^2 \; x_1x_2 \; x_2^2 \; x_1^3 \; x_1^2x_2 \; x_1x_2^2 \; x_2^3] \in \R^9$ is a vector that aggregates the monomials of the state, and each of $f_1, f_2, \cdots, g_{22}$ are randomly generated coefficient vectors in $\R^9$. We introduce model uncertainty to the true plant by perturbing the coefficient vectors from the nominal model.

In this example, we aim to design a control policy that stabilizes the system to the zero equilibrium point. 
To achieve this by using the certifying filter, we design the CLF $V(x) = x^\top Px$ as the certificate function, where $P$ is the solution of the Algebraic Riccati Equation for the linearized system of the nominal model \eqref{eq:polysys_dynamics} around $x=0$. 
\revision{Notably, the random coefficients for $x_1$ and $x_2$ in $f_1$ and $f_2$ are not perturbed, ensuring that the designed CLF is locally valid for both the nominal model and true plant after linearization.}
We set $\uref(x)\equiv0$ in this example. As shown in Figure \ref{fig:polysys_result}, we can see that the oracle CLF-QP (\textblue{blue}) is able to stabilize the state to the equilibrium, confirming that the CLF is a valid certificate function for the true plant. However, due to the model uncertainty we introduce to the true plant, the nominal model-based CLF-QP (\textmagenta{magenta}) fails to stabilize the system.

We next show the application of the GP-CLF-SOCP in \eqref{eq:gp-cbf-socp} on this example. We first construct the dataset $\dataset$ in order to apply GP regression to $\Delta$. We partition the subspace of the state space $[-2, 2] \times [-2, 2]$ into the coarse state grid of size $(10, 10)$. At every vertex $x_{j}$ of the state grid, we apply the randomly sampled control input $u_{j}$ to simulate the system \eqref{eq:polysys_dynamics} for a sampling time $\Delta t$ and collect a single data point $(\gpvaraug_j = (\gpvar_j, u_j),\bar{z}_j)$. We account for the numerical differentiation error in obtaining $\bar{z}_j$ as measurement noise.
In addition to the data from the coarse grid, we also incorporate some densely populated data points centered at a selected state and action pairs, $(x_{a}, u_{a})$. Around this point, a dense state-control grid is created by gridding up $[x_{a, 1}\!-\!\psi, x_{a, 1}\!+\!\psi]$~$\!\times [x_{a, 2}\!-\!\psi, x_{a, 2}\!+\!\psi]$~$\!\times [u_{a, 1}\!-\!\psi, u_{a, 1}\!+\!\psi]$~$\!\times [u_{a, 2}\!-\!\psi, u_{a, 2}\!+\!\psi]$ in (2, 2, 2, 2) grid, where we set $\psi=0.1$.
This results in a total of 80 data points collected additionally. In the subsequent sections describing the Polysys example, we refer to the data points generated from a single dense grid as the \textit{data cluster}. Combined together, we get in total $N\!=\!180$ data points, visualized by their projection to the state space in Figure \ref{fig:polysys_result} and \ref{fig:data_selection}.

We use GP regression to fit $\Delta$ from the dataset presented above, using the ADP compound kernel with isotropic squared exponential kernels as components. Then, we apply the GP-CLF-SOCP of \eqref{eq:gp-cbf-socp} to control the system. As shown in Figure \ref{fig:polysys_result}, the GP-CLF-SOCP using the full dataset (black dashed line) is able to stabilize the system to near the origin despite the uncertainty in the true plant dynamics. 

\begin{figure}
\centering
\includegraphics[width=.8\columnwidth]{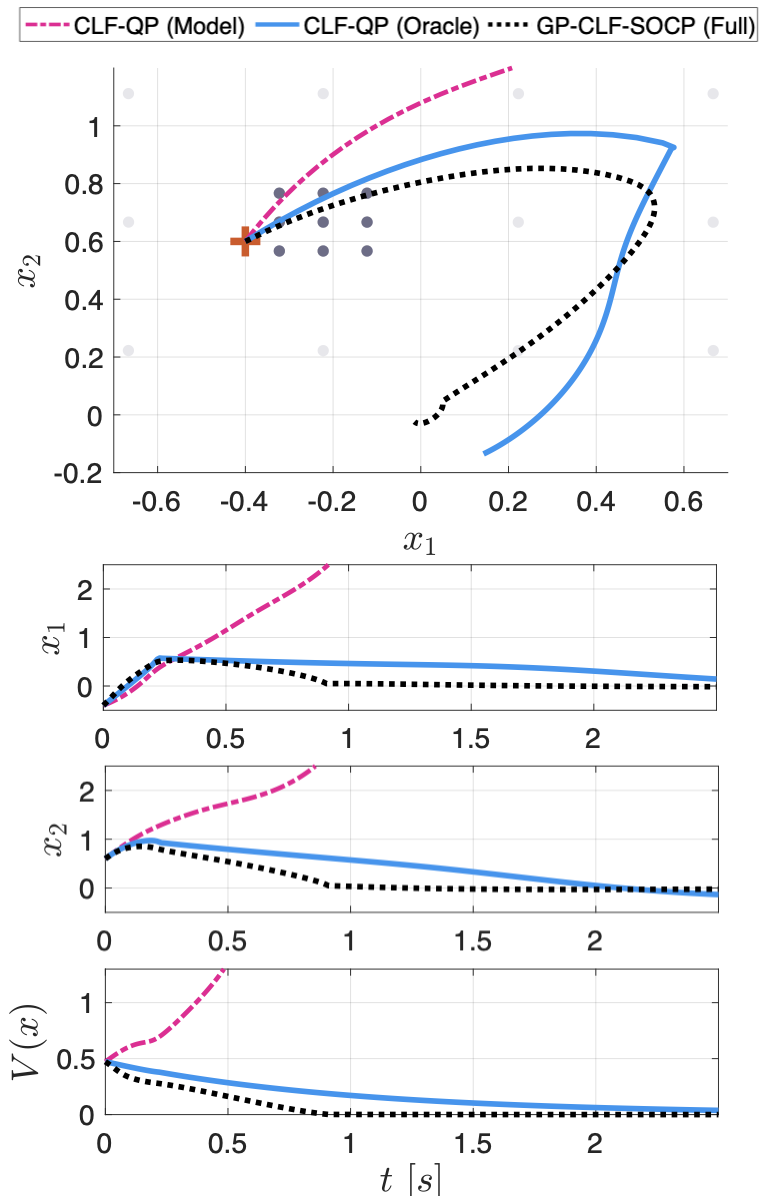}
\caption{\revision{The simulation result of the Polysys example under various controllers:} the nominal model-based CLF-QP (\textmagenta{magenta}), the oracle CLF-QP (\textblue{blue}), the GP-CLF-SOCP using full data (black).
The topmost plot illustrates the trajectory's progression in the state space for 2.5 seconds, with an initial state of $x_0=[-0.4\; 0.6]^\top$, while the data is depicted as grey dots. The three subplots on the bottom show the state $x_1$, $x_2$, and the CLF values respectively. While the trajectory quickly diverges under the nominal model-based CLF-QP, the GP-CLF-SOCP successfully stabilizes the trajectory to the origin.}
\label{fig:polysys_result}
\end{figure}
\section{Constraint Guided Online Data Selection}
\label{sec:sparse_gp}

In this section, we present the core contribution of our paper: a constraint-guided online data selection algorithm that improves the time complexity of GP inference for the GP-CF-SOCP from $\mathcal{O}(N^2)$ to $\mathcal{O}(N)$.

Using the entire dataset to evaluate $\sigma^{2}(\xuquery|\dataset)$ would yield minimal uncertainty for any query point $\xuquery$, as we would utilize all available information, but this comes at the cost of high computational demand. One way to mitigate this computational burden involves constructing an offline model that approximates the precise Gaussian Process (GP) inference, with the goal of making accurate predictions for any new query points encountered during runtime. \revision{However, our approach, similar to many existing Sparse GP methods \cite{quinonero2005unifying, seo2000gaussian}, is based on the idea that it is not necessary to reduce the uncertainty globally.} Instead, we aim to reduce the uncertainty for specific input classes relevant to our problem. The former approach, known as \textit{induction}, aims to regress the function with high quality across the entire input space. In contrast, our approach, which is called \textit{transduction}, focuses on learning only for specific test points that we care about. \revision{Revisiting the learning objective in our problem, we seek to ensure that the certifying chance constraint \eqref{eq:socp-cbf-constraint} is feasible, so that we can find a feasible $u^{*}(x)$ in \eqref{eq:gp-cbf-socp}.} Therefore, our data selection algorithm is designed to efficiently achieve this goal.

To facilitate the presentation of our algorithm, we first introduce some simplified notations and preliminaries that will be used in this section. We also present a sufficient condition for the feasibility of GP-CF-SOCP, from which we derive the main control input direction we want to characterize. This control input direction is the foundation upon which we apply the concept of transduction in our data selection algorithm.

\subsection{Preliminaries}

\subsubsection{Simplified notations for kernels}
\label{subsec:simplified_notations}
We use
\begin{equation*}
\kadp_{ij} := \kadp
\left((\gpvar_i, u_i), (\gpvar_j, u_j)\right),
\end{equation*}
\begin{equation*}
    \kadp_{**}(x, u) : = \kadp\left((x, u), (x, u)\right),
\end{equation*}
\begin{equation*}
    \kadp_{*i}(x, u) : = \kadp\left((x, u), (\gpvar_i, u_i)\right),
\end{equation*}
and $\kadp_{i}:=\kadp_{ii}$ as simplified notations, where $(\gpvar_i, u_i)$ is an input point in $\dataset$. We also consider the compound kernel that captures only the control vector field-relevant part:
\begin{align}
    \kadpu \! \left((x, u), (x', u')\right) :=  u^\top \! \Diag(k_{g_1}(x, x'),\!\cdots\!, k_{g_{m}}(x, x')) u'.
\label{eq:adpkernel_u}
\end{align}
Note that $\kadp \! \left((x, u), (x', u')\right) = k_f(x, x') + \kadpu\left((x, u), (x', u')\right)$ from Definition \ref{def:adpkernel}. Similarly, we define
\begin{equation*}
    \kadpu_{**}(x, u) : = \kadpu\left((x, u), (x, u)\right),
\end{equation*}
\begin{equation*}
    \kadpu_{*i}(x, u) : = \kadpu\left((x, u), (\gpvar_i, u_i)\right).
\end{equation*}

\subsubsection{Alternative expression for the GP posterior variance \eqref{eq:sigma_adp}}
Using simplified notations, we can express
\begin{equation*}
    \yqueryrow \KqueryY = \left[ \kadp_{*1}(\xquery, \uquery) \; \cdots \; \kadp_{*N}(\xquery, \uquery) \right],
\end{equation*} 
and \eqref{eq:sigma_adp} becomes
\begin{align}
\label{eq:sigma_adp2} 
\sigma^{2}&(\xuquery|\dataset)\\
& = \kadp_{**}(\xquery, \uquery)  - \left[ \kadp_{*1} \; \cdots \; \kadp_{*N} \right] (\Kadp + \sigma_n^2 I )^{-1} \begin{bmatrix} \kadp_{*1} \\ \vdots \\ \kadp_{*N},
\end{bmatrix} \nonumber
\end{align}
with $(\xquery, \uquery)$ dropped in $\kadp_{*i}$ for simplicity. Note that the first term on the right-hand side is contributed from the GP prior, and the choice of the data only affects the second term.

\subsubsection{Sufficient Condition for Pointwise Feasibility of GP-CF-SOCP}

The expression of the certifying chance constraint in \eqref{eq:socp-cbf-constraint2} highlights the \revision{comparison} required to evaluate its feasibility, between the prediction uncertainty of the GP regression on the left-hand side and the mean-estimate of the true certifying constraint on the right-hand side. This structure is useful for verifying the following sufficient condition for the feasibility of \eqref{eq:socp-cbf-constraint2}.

\begin{lemma} 
\label{lemma:suffcient_cond} Given a dataset $\mathbb{D}_N$, for a point $x \in \mathcal{X}$, If there exists a constant $\alpha > 0$ such that the following inequality holds,
\begin{equation}
\label{eq:suffcient_cond}
    \beta \; \gpsigmaB\!\left(x, \alpha \socc^\top\big|\mathbb{D}_N\right) < \alpha \norm{\socc}^2
\end{equation}
then the GP-CF-SOCP in \eqref{eq:gp-cbf-socp} is feasible. The feasible control input can be found by taking $u = \alpha' \socc^\top$ with sufficiently large $\alpha' > 0$.
\end{lemma}

\begin{proof} See Appendix \ref{appendix:lemma1proof}.
\end{proof}

The main implication of the above lemma is that the feasibility of \eqref{eq:gp-cbf-socp} can be assessed by examining the size of the prediction uncertainty, $\sigma$, in just one control input direction, specifically the direction of the mean-based estimate of $\LgC(x)$, denoted as $\socc$. This direction is particularly important because according to what the mean prediction of the GP tells, it is the control input direction in which we can most effectively regulate the value of $\cert(x)$. If the prediction uncertainty is sufficiently small in this direction, by taking the control input in this direction with large enough magnitude, we can ensure \eqref{eq:gp-cbf-socp} to be feasible. 

\subsection{Data Selection Objective}
We seek to design an online data selection algorithm, that selects a subset of data from the entire dataset, $\datasetonline(x) \!\subset\!\dataset$, at every sampling time at the current state $x$. Once $M$ online data points are determined, the GP-CF-SOCP in \eqref{eq:gp-cbf-socp} is solved with the online dataset $\datasetonline(x)$ in place of $\dataset$, to determine the filtered control input $u^*(x)$ which will be applied to the system next. Among the data points in the full dataset, we want to select a limited number of points that are most helpful in characterizing the control direction that secures the feasibility of the certifying chance constraint in \eqref{eq:socp-cbf-constraint}. 

We attempt to achieve this by utilizing the result of Lemma \ref{lemma:suffcient_cond}, trying to make sure that condition \eqref{eq:suffcient_cond} is met with the limited $M$ data points we are allowed to use. Adopting the approach of transduction, the goal of the data selection is to reduce the uncertainty in the direction of $\socconline^\top$, i.e., select $\datasetonline(x)$ which best reduces $\gpsigmaB\!\left(x, \alpha \socconline^\top|\mathbb{D}_M\right)$ for sufficiently large $\alpha$. However, we do not know how large $\alpha$ needs to be to render \eqref{eq:suffcient_cond} feasible prior to selecting $\datasetonline(x)$ and actually solving the SOCP. Therefore, we eliminate the dependency on the magnitude of $\alpha$ by considering the following problem:
\begin{equation}
\label{eq:data_selection_objective_uncertainty}
    \arg \min_{\datasetonline(x)} \lim_{\alpha\rightarrow \infty }\frac{1}{\alpha}\gpsigmaB\!\left(x, \alpha \socconline^\top|\datasetonline\right).
\end{equation}

Note that we drop the dependency of $\datasetonline$ on $x$ whenever it is obvious, for notational simplicity. From the expression of the variance in \eqref{eq:sigma_adp2}, we can derive the following lemma that transforms the objective function above into a form without the appearance of $\alpha$:
\begin{lemma} \label{lemma:data_select_objective}
The optimization problem \eqref{eq:data_selection_objective_uncertainty} can be equivalently expressed as
\begin{equation}
\label{eq:optimization_subset_problem}
\arg \max_{\datasetonline\subset \dataset} \dataselectobjective(x, \socconline^\top),
\end{equation}
where $\dataselectobjective(x, u):=$
{\small
\begin{align}
\label{eq:data_select_objective}
\left[ \kadpu_{*1}(x, u) \; \cdots \; \kadpu_{*M}(x, u) \right] (\Kadponline + \sigma_n^2 I )^{-1}\!\begin{bmatrix} \kadpu_{*1}(x, u) \\ \vdots\\ \kadpu_{*M}(x, u) \end{bmatrix},
\end{align}}

\noindent which is the second order term in the control input $u$ of the posterior variance $\gpsigmaB^2\!\left(x, u|\datasetonline\right)$.
\end{lemma}

\begin{proof} See Appendix \ref{appendix:lemma2proof}.
\end{proof}

Thus, we will consider $\dataselectobjective(x, \socconline^\top)$ as the \textit{objective function of the data selection algorithm}.

\begin{remark} \revision{We do not have access to $\socconline$ prior to determining $\datasetonline$. Thus, we replace $\socconline$ in \eqref{eq:optimization_subset_problem} with either $\socc$, where $\dataset$ is the entire dataset, or $\widehat{L_g \cert}(x | \datasetonline')$, where $\datasetonline'$ represents the dataset selected online at the previous time step. Note that $\socc$ is a better estimate of the groundtruth $L_g C(x)$ than $\socconline$ since it uses the entire dataset. Also, evaluating $\socc$ in \eqref{eq:mean-based-estimate} only requires the computation of $\gpmuB (x,u|\dataset)$  but not $\sigma(x, u|\dataset)$. Since $\mathbf{z}^{\top} (\Kadp + \sigma_n^2 I )^{-1}$ in \eqref{eq:mu_adp} can be precomputed offline, the time complexity of evaluating $\socc$ online is $\mathcal{O}(N)$.
However, when $N$ is very large, it may be impractical or computationally infeasible to evaluate $(\Kadp + \sigma_n^2 I)^{-1}$  offline due to the matrix inverse. In such cases, an effective approximation for $\socconline$ can still be achieved by using $\widehat{L_g \cert}(x | \datasetonline')$ instead.}
\label{remark3}
\end{remark}

\subsection{\revision{Kernel-based Alignment Measure \& Naive Approach}}
\label{subsec:naive}
Before we proceed to present the main algorithm of the paper, let's take a moment to build a better understanding of the data points we wish to include in $\datasetonline(x)$. To facilitate this discussion and simplify our thought process, consider a \revision{hypothetical} scenario where all data points in $\dataset$ are not correlated with one another, meaning that $\kadp_{ij} = \kadp\left((x_i, u_i), (x_j, u_j)\right) = 0$ for all $i\neq j$. Additionally, let's assume there is no noise in the data, so $\sigma_n=0$. \revision{Although such data would not exist in practice, it is considered to help build our intuition for understanding what data point would be included in \eqref{eq:optimization_subset_problem}.}

In this simplified case, $\Kadp + \sigma_n^2 I =\Diag(\kadp_1, \cdots, \kadp_N)$, and from \eqref{eq:data_select_objective} it holds that

\vspace{-10pt}
{\small\begin{align*}
    & \dataselectobjective (x, u) \\ 
    & = \left[ \kadpu_{*1}(x, u) \; \cdots \; \kadpu_{*M}(x, u) \right]\! \Diag\!\left(\!\frac{1}{\kadp_1}, \cdots, \frac{1}{\kadp_M}\!\right)\!\!\begin{bmatrix} \kadpu_{*1}(x, u) \\ \vdots\\ \kadpu_{*M}(x, u) \end{bmatrix} \\
    & = \left[ \frac{\kadpu_{*1}(x, u)}{\sqrt{\kadp_1}} \; \cdots \; \frac{\kadpu_{*M}(x, u)}{\sqrt{\kadp_M}} \right] \!\!\begin{bmatrix} \frac{\kadpu_{*1}(x, u)}{\sqrt{\kadp_1}} \\ \vdots\\ \frac{\kadpu_{*1}(x, u)}{\sqrt{\kadp_1}} \end{bmatrix} \!=\!\sum_{i=1}^{M} \left( \frac{\kadpu_{*i}(x, u)}{\sqrt{\kadp_i}} \right)^{\!2}.
\end{align*}}

\noindent We define a \textit{kernel-based alignment measure} as
\begin{equation}
\nqi(x, u) := \frac{|\kadpu_{*i}(x, u)|}{\sqrt{\kadp_i}} = \frac{|\kadpu\left((x, u), (\gpvar_i, u_i)\right)|}{\sqrt{\kadp
\left((\gpvar_i, u_i), (\gpvar_i, u_i)\right)}}, \label{eq:normalized_kernel_metric}
\end{equation}
which results in
\begin{equation}
\label{eq:objective_simple_case}
    \dataselectobjective\!\left(x, \socconline^\top\right)\!=\!\sum_{i=1}^{M} \nqi^2\left(x, \socconline^\top\right).
\end{equation}
Therefore, we can optimize $\dataselectobjective\!\left(x, \socconline^\top\right)$ simply by selecting $M$ points from $\dataset$ that exhibit maximum values of $\nqi\!\left(x, \socconline^\top\right)$. The time complexity of finding such points is $\mathcal{O}(N)$, which can be achieved using efficient algorithms, such as a Quickselect algorithm \cite{hoare1961algorithm}. 

Equation \eqref{eq:objective_simple_case} highlights that the alignment measure $\nqi\!\left(x, \socconline^\top\right)$ is the measure of the relevance of the data point $(x_i, u_i)$ to the feasible direction of the certifying chance constraint. Here, we offer a concise explanation of the geometric interpretation of this measure.

For kernels used in GP regression, note that the kernel value of two inputs, $k(\gpvar, \gpvar')$ can be interpreted as an inner product between the feature vectors of $\gpvar$ and $\gpvar'$, i.e. $k(\gpvar, \gpvar') = \phi(\gpvar) \cdot \phi(\gpvar')$ \cite{williams2006gaussian}. For the ADP kernel in Definition \ref{def:adpkernel}, denoting the feature vectors for individual kernels as $\phi_f$, $\phi_{g_1}, \cdots, \phi_{g_m}$, we can express the ADP kernel's feature vector as $\boldsymbol{\phi}(x, u):= [\phi_f(x)\;\phi_{g_1}(x)\;\cdots\; \phi_{g_m}(x)]\left[\!\!\begin{array}{c} 1 \\ u \\\end{array}\!\!\right]$. Consequently, we get 
\begin{equation*}
    \small
    \nqi\!\left(x, \socconline^\top\!\right)\!= \!\!\lim_{\alpha\rightarrow\infty} \!\frac{\left|\boldsymbol{\phi}(x, \alpha \socconline^\top) \cdot \boldsymbol{\phi}(x_i, u_i)\right|}{\alpha \sqrt{\boldsymbol{\phi}(x_i, u_i) \cdot \boldsymbol{\phi}(x_i, u_i)}},
    \normalsize
\end{equation*}
from \eqref{eq:normalized_kernel_metric}, where we get rid of the autonomous vector field relevant part from the numerator in \eqref{eq:normalized_kernel_metric} by taking the limit of $\alpha\!\rightarrow\!\infty$. Thus, $\nqi\!\left(x, \socconline^\top\right)$ captures how well the data point is aligned in the feature space of the ADP kernel with the feasible input direction.

In summary, the naive approach, which selects $M$ points with maximum values of $\nqi\!\left(x, \socconline^\top\right)$ from the dataset $\dataset$, optimally achieves the objective in \eqref{eq:data_selection_objective_uncertainty} under the ideal conditions of an uncorrelated dataset and absence of measurement noise. However, these assumptions are rarely satisfied in real-world datasets. In fact, data from actual systems often have a high correlation because sampled data points from trajectories are sequential and share similar properties due to their close proximity in time and space.

We use the Polysys example to highlight the failure of the naive approach in handling datasets that contain self-correlated data points. Our demonstration reveals that the naive approach may choose an unsuitable $\datasetonline$, rendering the SOCP filter infeasible. This limitation motivates the development of a more advanced data selection algorithm, which we present in the next subsection.

\vspace{1em}
\textit{\textbf{Running Example--Polysys (Cont'd):}}
As described in Section \ref{subsec:running_example_intro}, the dataset created for the Polysys example contains highly correlated data points, especially in a data cluster. Figure \ref{fig:data_selection} (a) illustrates a failure case of the naive algorithm. In the first row of Figure \ref{fig:data_selection} (a), we visualize the selected data points $\datasetonline(x)$ at a query state $x$ under various values of $M$. The second row represents the prediction uncertainty $\beta\sigma(x,u)$ in control-input space as an ellipse, and $\socconline$ as a dashed magenta line, thereby illustrating the competitive relationship between the left-hand side (ellipse) and the right-hand side (magenta line) of \eqref{eq:socp-cbf-constraint2}. Since the naive approach greedily selects the points that maximize $\nqi\!\left(x, \socconline^\top\right)$ without considering the correlation between them, the selected data points are sourced from the data cluster that is close to the query state. The effect of using such highly self-correlated data points as $\datasetonline$ is shown in the second row of the figure. It demonstrates that even after increasing the size of $M$ from 5 to 20, the uncertainty ellipse barely reduces its size, leading to the infeasibility of the SOCP. Clearly, selecting such concentrated data points barely provides additional information, which illustrates why the naive approach can fail.

\begin{figure}
\centering
\includegraphics[width=\columnwidth]{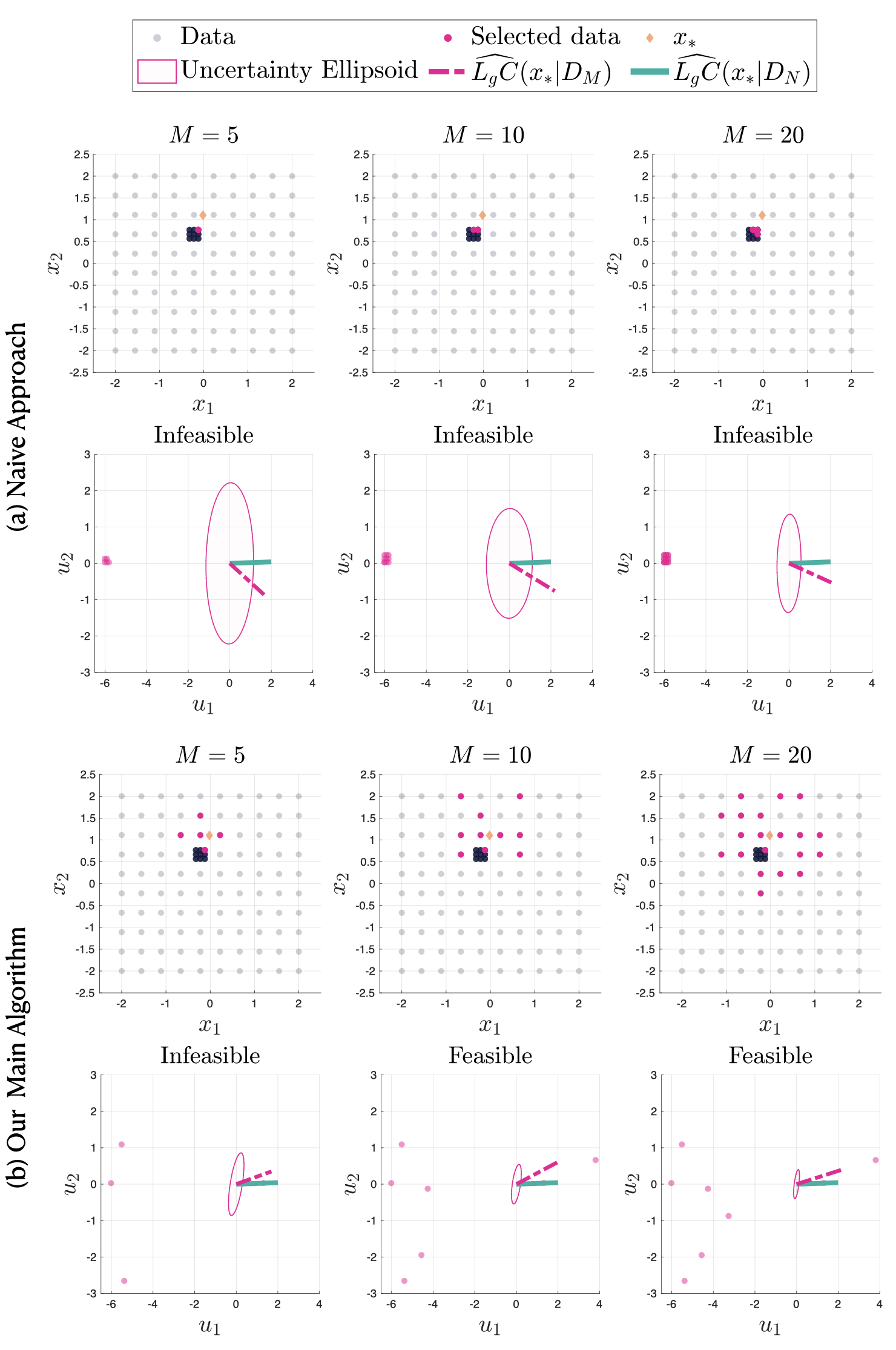}
\vspace{-1em}
\caption{Comparison between the two data selection strategies--(a) naive approach described in Section \ref{subsec:naive} and (b) our main algorithm described in Section \ref{subsec:coca}, on the Polysys running example system, with a varying number of online selected data points ($M=5, 10, 20$). In each case, the first row visualizes the entire dataset $\mathbb{D}_N$ (grey dots) projected on the state space and the data points selected online $\mathbb{D}_M$ (magenta dots) according to the data selection algorithm at the query state $x = [-0.02 \; 1.10]^\top$ (orange diamond). The second row visualizes the selected points projected on the control input space (magenta dots), and the prediction uncertainty $\beta\sigma(x, u)$'s growth in the control space as an ellipse. We also visualize $\socc$ and $\socconline$ as the dashed \textteal{green} and \textmagenta{magenta} lines, respectively. The ellipse and magenta line represent the growth of the right-hand side and left-hand side of \eqref{eq:socp-cbf-constraint2}, respectively. The feasibility of the chance certifying constraint can be 
deduced by evaluating the relative ratio of the magenta line's length to the ellipse's radial distance in the magenta line's direction. A smaller ratio suggests that a larger control input in the $\socconline$ direction is required to satisfy the chance constraint.} 
\label{fig:data_selection}
\end{figure}

\subsection{Main Algorithm}
\label{subsec:coca}

Selecting the data points in the dataset $\dataset$ that maximize our objective function $\dataselectobjective(x, \socconline^\top)$ when the dataset is self-correlated is in fact a combinatorial optimization problem which is NP-hard \cite{krause2008optimizing}.
This complexity occurs from the need, as seen in \eqref{eq:data_select_objective}, to find the optimal subset of data that maximizes the correlation with the target point ($\kadp_{*j}$ in \eqref{eq:data_select_objective}), while minimizing the self-correlation within the subset (captured by $K_{\datasetonline}$). Therefore, directly optimizing for the objective function online is intractable. The result presented next, which is the main assertion of our paper, allows us to indirectly find a good candidate $\datasetonline$ by maximizing a lower bound of the objective function.

\begin{theorem}
\label{thm:main} 
For a given dataset $\mathbb{D}_M$ with $M \ge 2$, assume that there exists a constant $\epsilon  \in [0, 1)$ that satisfies 
\begin{equation}
\label{eq:correlation}
    \kadp_{ij}^2 < \epsilon^2 \kadp_{i} \kadp_{j},
\end{equation}
for all $i, j = 1, \cdots, M$ and $i\neq j$, and
\begin{equation}
\label{eq:noise_cond}
    \sigma_n^2 \le \frac{\epsilon^2 (M-1) \min_{i} \kadp_{i}}{1-\epsilon}.
\end{equation}
Then, $\dataselectobjective(x, u)$ is lower bounded by the inequality below
\begin{equation}
\label{eq:main-lower-bound}
\dataselectobjective(x, u) \ge \frac{1 - \epsilon}{1 + \epsilon(M-2)} \sum_{i=1}^{M} \nqi^2(x, u).
\end{equation}
Note that the equality is satisfied when $\epsilon = 0$.
\end{theorem}

\begin{proof}
See Appendix \ref{appendix:thm1proof}.
\end{proof}

\revision{Theorem \ref{thm:main} concludes that the lower bound of the objective function can be maximized by selecting $M$ points with maximum values of $\nqi\!\left(x, \socconline^\top\right)$, under conditions \eqref{eq:correlation} and \eqref{eq:noise_cond}. The condition \eqref{eq:correlation} requires the dataset to exhibit no more than a weak correlation specified by $\epsilon$, while condition \eqref{eq:noise_cond} necessitates that the noise variance remains comparatively small with respect to $\epsilon$ and $M$.

\begin{remark} \textit{(When the noise level is high)} Highly correlated data can often be useful, especially when the measurement noise is large, since repetitive measurements at the same input can help reduce the GP posterior variance \cite{calandriello2022scaling}. We provide an alternative lower bound of $\dataselectobjective(x, u)$ in Appendix \ref{appendix:alternative-theorem}, similar to but different from Theorem \ref{thm:main}, which does not require an upper bound on $\sigma_n$ like \eqref{eq:noise_cond}, by using an alternative kernel-based alignment measure that includes the noise variance. Both lower bounds require $\epsilon\!<\!1$ since with $\epsilon\!=\!1$, our lower bound becomes trivial as the right-hand side of \eqref{eq:main-lower-bound} is zero. Setting $\epsilon\!=\!1$ might be desirable when the noise is high and there exist multiple data points at exactly the same state and input. However, in practice, any realistic dataset collected from real-world trajectories would not contain such data points. 

It is worth noting that condition \eqref{eq:noise_cond} becomes less stringent as $M$ increases, implying that if we use more data, we can tolerate larger measurement noise. Conversely, if $M$ is constrained to be small, we need to allow more correlation in data (larger $\epsilon$) so that the inference can benefit from multiple noisy measurements of similar data points. An experiment conducted in Appendix \ref{appendix:noise-experiment} shows that as the noise level increases, choosing a larger $\epsilon$ helps achieve a higher $\dataselectobjective(x, u)$.
\end{remark}
} 

Leveraging the result of Theorem \ref{thm:main}, we aim to maximize the lower bound as a proxy for the original objective function, thereby rendering the problem more tractable. The essence of our main algorithm is to condition the dataset to satisfy \eqref{eq:correlation}, ensuring that Theorem \ref{thm:main} holds, and then identify the data points for which $\sum_{i=1}^{M} \nqi^2(x, u)$ is maximized. 

We achieve this through a two-fold algorithm. First, during the offline phase, we compute a ready-to-use binary \revision{correlation indicator} matrix $B \in \R^{N \times N}$, with elements defined as follows:
\begin{equation}
\label{eq:binary_matrix}
B_{ij} := \begin{cases}
    1 & \text{if } \kadp_{ij}^2 \revision{\ge} \epsilon^2 \kadp_{i} \kadp_{j}, \\
    0 & \text{otherwise.}
\end{cases}
\end{equation}
It can be efficiently constructed by checking the matrix
\begin{equation}
{\Diag\!\left(\!\frac{1}{\sqrt{\kadp_1}}, \cdots, \frac{1}{\sqrt{\kadp_N}}\!\right)\!\!}^\top \Kadp {\Diag\!\left(\!\frac{1}{\sqrt{\kadp_1}}, \cdots, \frac{1}{\sqrt{\kadp_N}}\!\right)\!},
\end{equation}
exceeding the threshold $\epsilon$.
This operation has a time complexity of $\mathcal{O}(N^2)$ but occurs during the offline stage, so it does not impact the online time complexity.

Next, in the online phase described in Algorithm \ref{algo:sparsegp}, we first initialize a candidate dataset as the entire dataset (Line 5-6), \revision{from which we select the online data that initially starts as the empty set..} We then sequentially add to the online dataset $\datasetonline$ the data point that has the maximum value of $\nqi\!\left(x, \socconline^\top\right)$ among those in the candidate dataset (Line 8-11). As we select each point, we remove from the candidate dataset the points that have a correlation greater than $\epsilon$ relative to the selected point, by directly referring to the matrix $B$ (Line 12-13). 

Algorithm \ref{algo:sparsegp} has a time complexity of $\mathcal{O}(MN)$, as each operation in Line 6 and Line 8 inside the for loop is $\mathcal{O}(N)$. At each time step, after obtaining $\datasetonline$ from the proposed algorithm, we use this online dataset for the GP-CF-SOCP filter in \eqref{eq:gp-cbf-socp} instead of using the entire dataset. This requires evaluating the matrix inverse in \eqref{eq:mu_adp} and \eqref{eq:sigma_adp} online, which has a time complexity of $\mathcal{O}(M^3)$. Thus, with our proposed approach, obtaining the optimal filtered control input $u^*(x)$ has a total time complexity of $\mathcal{O}(NM + M^3)$, in terms of $N$ and $M$.
\revision{The time complexity of our online data selection algorithm is linear in $N$, and we choose $M << N$ in practice. Note that the time complexity of solving the SOCP does not depend on the number of data points.}

\revision{\begin{remark} \textit{(Choosing the right values of $\epsilon$ and $M$)}
Choosing the value of the correlation threshold $\epsilon$ allows users to strike a balance between the contribution of the term $\sum_{i=1}^{M} \nqi^2(x, u)$ and the impact of self-correlation on the objective function $\dataselectobjective(x, u)$.
Choosing $\epsilon$ close to one loosens the lower bound in \eqref{eq:main-lower-bound}; however, an excessively small $\epsilon$ prohibits users from using data points with even the slightest correlation, which is impractical. Ideally, we should find the optimal value of $\epsilon$ that offers the best trade-off. As shown in the experiment in Appendix \ref{appendix:noise-experiment}, this optimal value might also depend on the noise level of the data. Determining the optimal $\epsilon$ is an NP-hard problem, as it shares the same problem complexity as maximizing the data selection objective $\dataselectobjective$ directly. A practical and effective strategy is to leverage prior knowledge of the full dataset to identify an acceptable $\epsilon$ value, for instance, by evaluating the histogram of $\frac{\kadp_{ij}^2}{\kadp_{i} \kadp_{j}}$ for the dataset and selecting an $\epsilon$ that corresponds to a reasonable quantile of data satisfying \eqref{eq:correlation}. 

Similarly, careful attention is required when choosing the value of $M$, to achieve the right balance between the sufficiency of information in $\datasetonline$ and the computation time, which will increase with respect to $M$. An effective strategy would be to dynamically update $M$, by ending the for loop of Algorithm \ref{algo:sparsegp} when condition \eqref{eq:suffcient_cond} is satisfied, which ensures the feasibility of the SOCP. However, in our experiments, it was sufficient to set $M$ to fixed values.
\end{remark}}

\vspace{1em}
\textit{\textbf{Running Example--Polysys (Cont'd):}} We investigate how Algorithm \ref{algo:sparsegp} selects data online and improves the downstream objective of enhancing the feasibility of the GP-CLF-SOCP through its self-correlation remedy in the Polysys example. \revision{We chose $\epsilon=0.95$ as the correlation threshold in this example. The correlation within the data cluster is greater than this value; thus, using this value prevents our main algorithm from selecting more than one point per data cluster.} The first row of Figure \ref{fig:data_selection} (b) displays that our main algorithm selects at most one data from each concentrated data cluster even as $M$ increases. This correlation-aware behavior resulting from upper-bounding the maximum self-correlation of the selected data points induces the algorithm to select diverse data. Consequently, the prediction uncertainty, illustrated as the ellipse in the second row of the image, is reduced as $M$ increases in all directions of $u$, but more importantly, it is primarily reduced in the direction of $\socc$. Moreover, in the case of $M=20$, it is notable that the algorithm prioritizes selecting data points whose control input values are well aligned in the direction of $\socconline$. 
As a result, the GP-CLF-SOCP controller utilizing the online dataset constructed by our main algorithm attains feasibility even in the small online dataset size as shown for $M=10, 20$ in Figure \ref{fig:data_selection}.



\begin{algorithm}[t]
\small
\DontPrintSemicolon 
\textbf{Input: } Current state $x$, entire dataset $\dataset$, 
$B$ defined in \eqref{eq:binary_matrix}
\\ \textbf{Output: } Online dataset $\datasetonline$ 
\\
\revision{// Initialize online dataset}
\\
$\datasetonline \gets \emptyset$\;
\revision{// Index set of candidate data \;}
$\mathcal{I}_{\text{candidate}}\gets\{1, 2, \cdots, N\}$ $\quad$ \;
\For{$k=0; ~k<M; ~k=k+1$}
{
    \revision{// Find the index of data that maximizes \eqref{eq:normalized_kernel_metric} \;}
    $i_* \gets \argmax_{i\in \mathcal{I}_{\text{candidate}}}~ n_i(x, \socconline^\top)$\;
    \revision{// Add  $i_*^{\text{th}}$ data to the online dataset \;}
    $\datasetonline \gets \datasetonline \cup (x_{i_*}, u_{i_*}, z_{i_*}) $\;
    \revision{// Remove data correlated with $i_*^{\text{th}}$ data from candidates \;}
$\mathcal{I}_{\text{candidate}}\gets\mathcal{I}_{\text{candidate}}\setminus \{j\in \mathcal{I}_{\text{candidate}} \; | \; \revision{B_{i_* j} ==1}\}$\;
}
\caption{Online Data Selection for GP-CF-SOCP}
\label{algo:sparsegp}
\end{algorithm}

\subsection{Related Data Selection Methods}
The point at issue of this paper is very related to the information-theoretic data subset selection \cite{daszykowski2002representative, wei2015submodularity} and sensor placement \cite{currin1991bayesian, krause2008optimizing} problems, which are known to be NP-hard for many different objective functions, such as mutual information and conditional entropy \cite{ko1995exact, krause2008near}. While our focus is on optimizing a particular certification-oriented measure \eqref{eq:data_select_objective} that differs from the information-theoretic objective functions, our optimization problem still suffers from the same combinatorial challenges, and solving \eqref{eq:optimization_subset_problem} to optimality would be intractable for large datasets. 

A reasonable alternative to our approach would be to form the online dataset $\datasetonline$ by greedily selecting, one at a time, the data points that maximize \eqref{eq:optimization_subset_problem}. This idea was applied to the sensor placement in \cite{krause2008near} and has been used for data-driven control in \cite{umlauft2020smart, lederer2021impact}. To approximately solve \eqref{eq:optimization_subset_problem}, this greedy selection method can be implemented with an asymptotic time complexity of $\mathcal{O}(NM^3)$. While this asymptotic complexity is only slightly worse than the $\mathcal{O}(NM)$ complexity of Algorithm \ref{algo:sparsegp}, in practice we observe that the greedy method is too slow to perform the data selection online, even when using the locality and lazy evaluation speedups proposed in \cite{krause2008optimizing}. 

Another simpler approach would be to choose the $k$-nearest neighbors (k-NN) at each query state-action pair. However, given the control-affine structure of the target function $\Delta$, it is not immediately clear which distance metric should be used for the k-NN to capture the most relevant information. The authors in \cite{yu2002kernel} propose to use the \emph{kernel distance} \cite{hein2005hilbertian,phillips2011gentle}, which is the Euclidean distance in the kernel feature space. Although simple, these k-NN selection approaches suffer from similar problems as our naive selection algorithm, as they do not consider the self-correlation of the dataset.
\section{Results}
\label{sec:results}

In this section, we apply our method to three specific examples, consisting of two numerical simulations and one hardware experiment. We refer to the GP-CF-SOCP filter using the full dataset as \textit{\textbf{GP-CF-SOCP (Full)}}, the GP-CF-SOCP using the online data constructed by the naive approach in Section \ref{subsec:naive} as \textit{\textbf{GP-CF-SOCP (Naive)}}, and the GP-CF-SOCP using the online data constructed by our main data selection algorithm as \textbf{\textit{GP-CF-SOCP (Ours)}}.

\revision{Finding the right values of $\beta$ and kernel hyperparameters that satisfy Definition \ref{asump:well-calibrated} is a non-trivial task, as noted by existing works \cite{berkenkamp2023bayesian, capone2022gaussian}. In our experiments, we initialize the hyperparameters by maximizing the log marginal likelihood \cite{williams2006gaussian}. Then, we decide the value of $\beta$ and fine-tune the parameters by evaluating the satisfaction of \eqref{eq:well_calibrated_assumption} on the training data, with the probability threshold $\delta = 0.01$.}

\vspace{-0.5em}
\subsection{Running Example: Polynomial System (Cont'd)}
\begin{figure}
\centering
\includegraphics[width=.8\columnwidth]{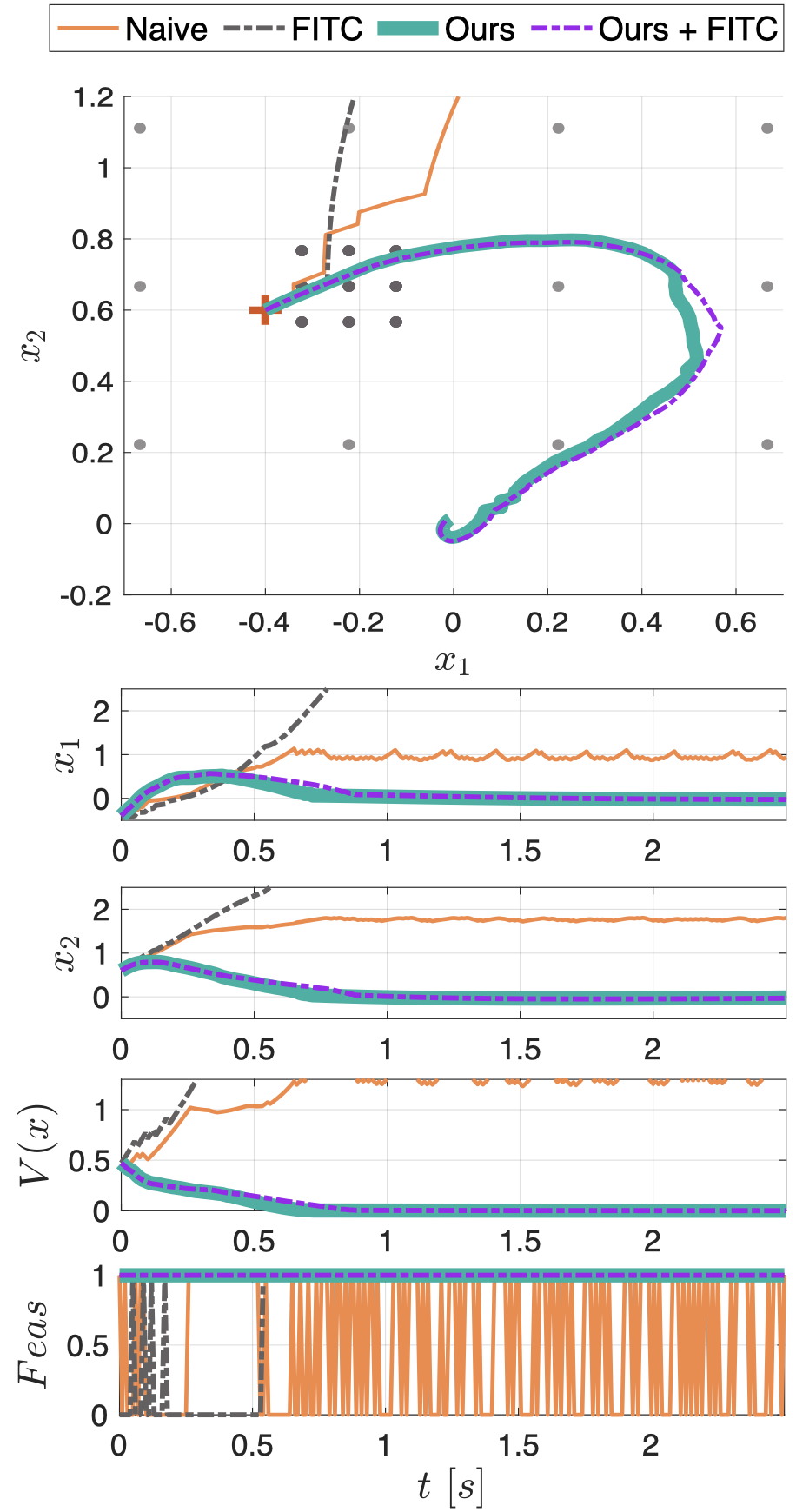}
\caption{\revision{The simulation result of the Polysys example under the GP-CLF-SOCP controllers using:} naive data selection (\textorange{orange}) discussed in Sec. \ref{subsec:naive}, FITC \cite{snelson2005sparse} whose inducing points are evenly spaced and fixed (\textgrey{grey}), dataset selected from Algorithm \ref{algo:sparsegp} (\textteal{green}) discussed in Sec. \ref{subsec:coca}, and FITC whose inducing points are selected by Algorithm \ref{algo:sparsegp} (\textpurple{purple}). All controllers use the same number of online data or the number of inducing points, $M=20$.
While the naive approaches often face infeasibility and fail to stabilize the system close to the origin, our approaches effectively select an online dataset that secures the feasibility of the SOCP.}
\vspace{-1em}
\label{fig:polysys_sim_result_gp}
\end{figure}

The simulation results of our method on the Polysys example are evaluated, extending the analysis in Section \ref{subsec:running_example_intro}. \revision{In addition to GP-CLF-SOCP (Ours) and GP-CLF-SOCP (Naive), we consider SOCP controllers whose GP inference is approximated by FITC \cite{snelson2005sparse}, a standard Sparse GP method that uses inducing points to construct the approximation of the posterior GP. We consider two variations of FITC: first, where the inducing points are selected as fixed grid points across the state space, and second, where the inducing points are selected by our data selection algorithm online.
To ensure a fair comparison, all controllers select 20 data points ($M$, or number of inducing points) from the full dataset constructed in Section \ref{subsec:running_example_intro}.
The results presented in Figure \ref{fig:polysys_sim_result_gp}, show that our method, both used in selecting $\datasetonline$ or in selecting inducing points for FITC, successfully stabilizes the closed-loop system. 
In contrast, the GP-CLF-SOCP (Naive) and FITC with naive inducing points fail to do so, and the SOCP is infeasible very frequently with these approaches. Note that when the SOCP is infeasible, the backup controller in \eqref{eq:backup_controller} is deployed.
} 

The Polysys example is devised to provide a detailed walk-through of our method; thus, we do not benchmark the computation time in this example. Given the relatively small number of data points in this example, the computational efficiency gained from our method would not be easily noticeable.

\vspace{-0.5em}

\subsection{High-dimensional System in Simulation: Five-link Walker}
\label{subsec:walking robot}

We explore the performance of our algorithm in a high-dimensional system, RABBIT \cite{chevallereau2003rabbit}, a planar five-link bipedal robot consisting of ten state variables. We demonstrate the effectiveness of our algorithm in achieving stable walking. The significance of our algorithm in reducing the computational demands of executing the certifying filter is highlighted.

RABBIT is a testbed system developed to study bipedal robot locomotion \cite{chevallereau2003rabbit}. As depicted in Figure \ref{fig:systems} (a), its configuration is represented by the generalized coordinate vector $q=[q_1, q_2, q_3, q_4, q_5]^\top$ consisting of the robot's joint angle variables. We adopt the mathematical model for RABBIT locomotion in \cite{chevallereau2003rabbit} to design the simulation model of this system, where the state is defined as $x = [q, \dot q] \in \R^{10}$, and the control input is defined as $u \in \R^4$, consisting of the hip and knee motor torques for both legs. The torque saturation is set at \qty{150}{\N\m}. The hybrid system description of the robot's walking process consists of a single-support swing phase under a Lagrangian dynamics and a reset map defined by the rigid impact model, which switches the robot's state to the post-impact state upon the swing foot's impact with the ground.

The objective of the certifying filter is to achieve an exponentially stabilizing periodic gait for RABBIT, despite the effect of the impacts. To accomplish this, we employ a Rapidly Exponentially Stabilizing Control Lyapunov Function (RES-CLF) \cite{ames2014rapidly} as our certificate function. We also set $\uref(x) \equiv 0$ since this naturally captures the objective of minimizing the energy spent to produce the motor torques. In order to construct RES-CLFs, we first input-output (IO) linearize the continuous dynamics of the system by defining the output functions, $y(q) = q_{2:5} - y_d(\theta(q))$,
where $\theta(q)$ is a variable that defines the phase along the gait, which monotonically increases within each walking step, and $y_d(\cdot)$ is a desired gait represented by a Bezier polynomial, generated offline using the Fast Robot Optimization and Simulation Toolkit (FROST) \cite{hereid2017frost}. We can then decompose the state of the system into the transverse coordinates $\xi\!=\![y \; \dot y]^\top\!\in\!\R^8$, and the zero coordinates $\eta\!=\![\theta(q) \; \dot \theta(q)]\!\in\!\R^2$. After applying the IO linearization, we can represent the transverse dynamics as:
\begin{equation}
\label{eq:iolinearizedsys}
\dot \xi = f(\xi, \eta) + g(\xi, \eta)\mu,
\end{equation}
where $\mu$ is the virtual input. By stabilizing $\xi$ to zero, we enforce the joint trajectory to converge to the desired stable walking gait defined by $y_d(\theta(q))$.

Model uncertainty is introduced in the simulation by scaling the mass and inertia values of the robot by a factor of 2, which poses a challenge for the controller to maintain stability during walking. Note that a payload is one of the most common sources of model uncertainty for legged robots in practical applications. As illustrated in Figure \ref{fig:rabbit_sim_result}, while the oracle CLF-QP (\textblue{blue}), which assumes access to the true plant dynamics, successfully completes fifteen steps, the nominal model-based CLF-QP (\textmagenta{magenta}), which is unaware of the change in mass and inertia, fails to stabilize the robot and it eventually falls down during the fourteenth step. This observation motivates the use of the GP-CLF-SOCP controller.

\begin{figure}[!t]
\centering
\subfloat[]
    \centering
\includegraphics[width=1.2in]{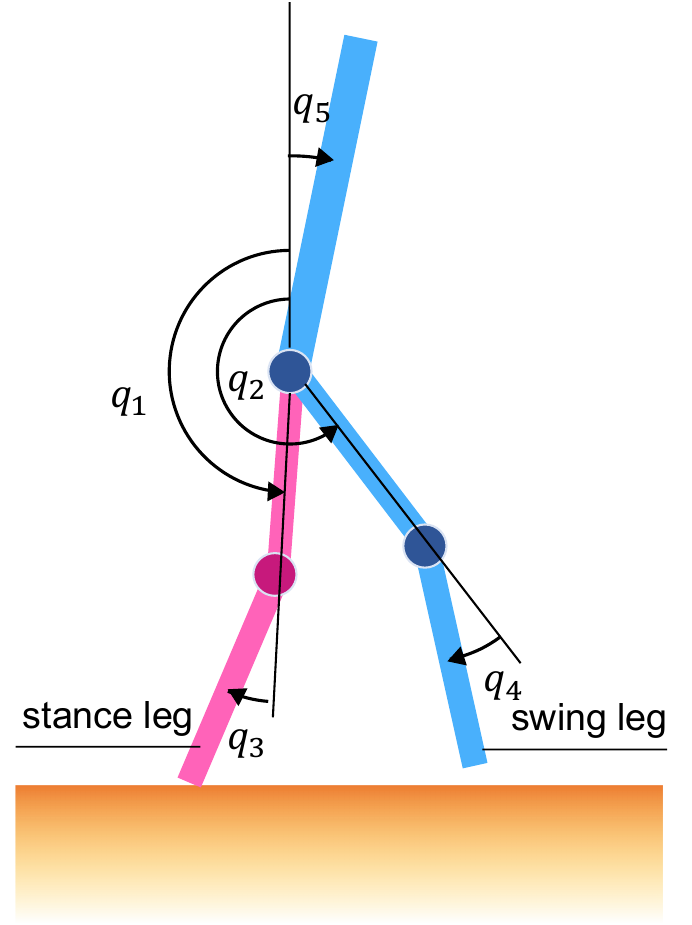}
    \label{fig:rabbit_state}
\subfloat[]
    \centering
\includegraphics[width=1.5in]{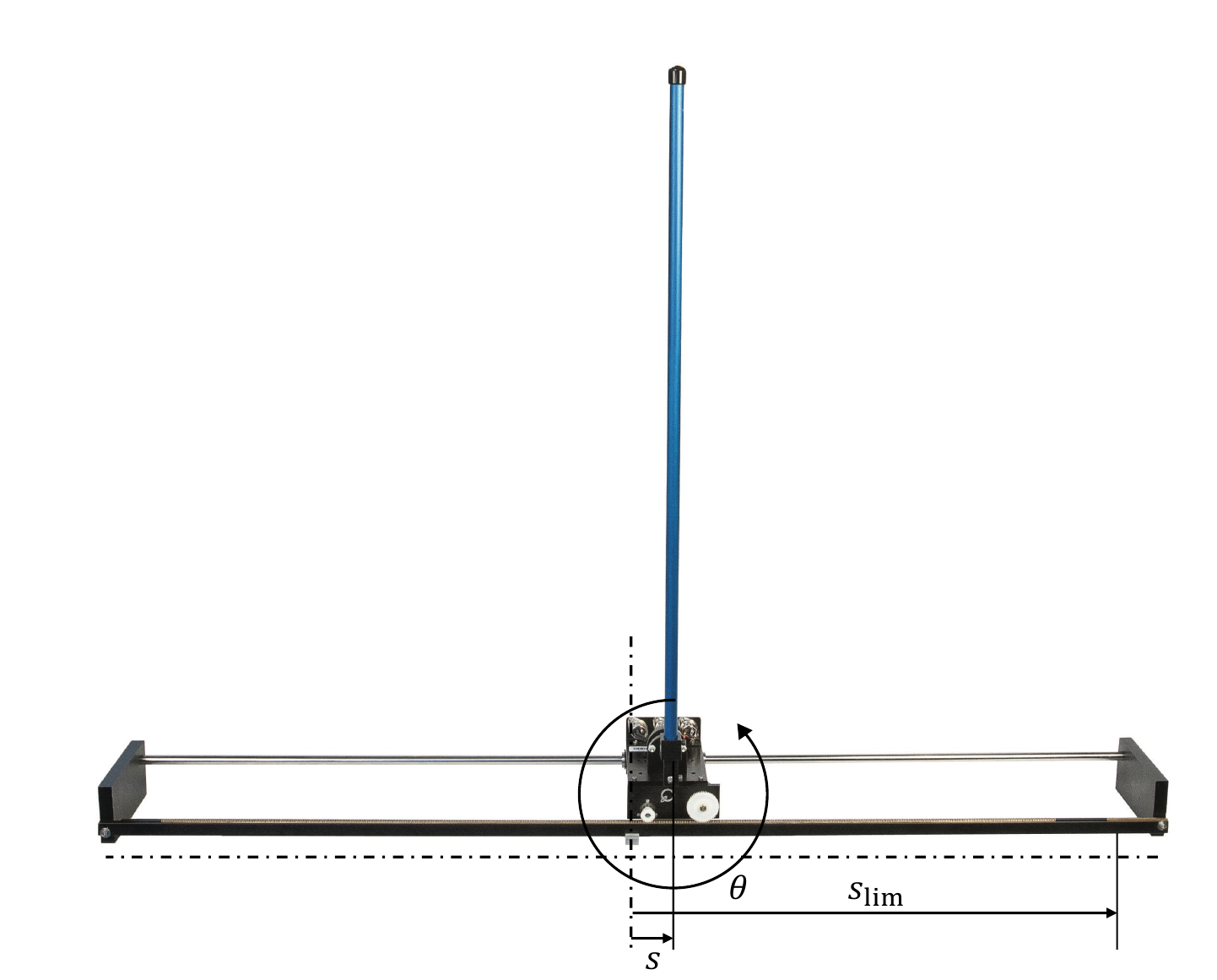}    \label{fig:cart_pole_diagram}
\caption{(a) The configuration of the planar five-link bipedal robot RABBIT \cite{chevallereau2003rabbit} (b) Cart-pole experiment setup based on Quanser Linear Servo Base Unit with Inverted Pendulum \cite{quanser_products_2021}.}
\vspace{-0.5em}
\label{fig:systems}
\end{figure}

\begin{figure*}
\centering
\includegraphics[width=.9\textwidth]{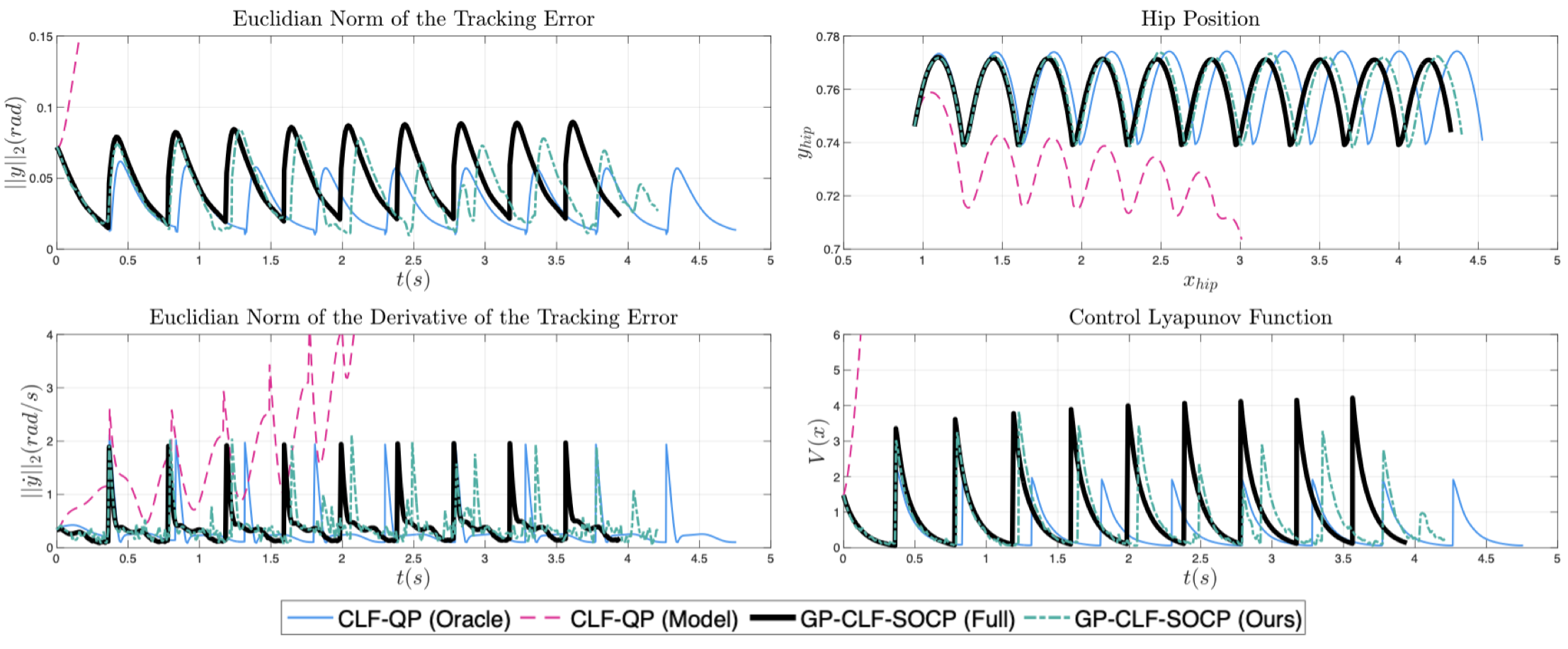}
\vspace{-0.5em}
\caption{Simulation results of RABBIT achieving stable walking under various controllers: the nominal model-based CLF-QP (\textmagenta{magenta}), the oracle CLF-QP (\textblue{blue}), GP-CLF-SOCP (Full) (black), and GP-CLF-SOCP (Ours) (\textteal{green}). The left column depicts histories of the Euclidean norm of the tracking error $y$ and its time derivative $\dot{y}$ with respect to the reference gait. 
The right column shows the evolution of the hip's vertical position from the ground and the value of the CLF $V(x)$.}
\vspace{-0.5em}
\label{fig:rabbit_sim_result}
\end{figure*}

We collect data points represented as $\gpvaraug_j\!=\!([\xi_j, \eta_j], \mu_j)$ since we aim to learn the effect of model uncertainty in the transverse dynamics \eqref{eq:iolinearizedsys}. The dataset is collected in an episodic learning fashion. The nominal model-based CLF-QP is run in the first episode to create an initial dataset. Following this, the GP-CLF-SOCP is executed, and the data collected from the new trajectory is iteratively added to the dataset. For the GP-CLF-SOCP, we initially use the full dataset; however, when the execution time of the SOCP controller approaches the limit of the target sampling time, we activate the data selection algorithm. As data is inherently more scarce in high-dimensional systems, to mitigate the out-of-distribution problem, we introduce perturbations to the reset map and create variations in the control policies executed in each episode, for example, by altering the number of $M$, in order to enhance the dataset's coverage. As a result, we obtain a comprehensive dataset comprising 12,765 ($N$) data points. 

When assuming the ability to deploy the GP-CLF-SOCP (Full) at a sampling rate of \qty{40}{\hertz} (\qty{25}{\ms}), it can achieve fifteen successful steps without falling, as shown in Figure \ref{fig:rabbit_sim_result} (black). However, this would not be achievable in reality, as the average execution time of the controller using the full dataset is \qty{238.3}{\ms}, which significantly exceeds the target sampling time. Instead, we employed our main data selection algorithm to choose 30 ($M$) data points from the full dataset. As demonstrated in Table \ref{table:results}, this algorithm significantly reduced the execution time to an average of \qty{22.9}{\ms}. 

As shown in Figure \ref{fig:rabbit_sim_result}, the GP-CLF-SOCP (Ours) (\textteal{green}) enables the robot to successfully complete fifteen steps. This is further evidenced by the CLF and tracking error plot, where the controller consistently and exponentially stabilizes the tracking error close to zero after the repeated state resets.
It is worth noting that the resulting walking gait of the GP-CLF-SOCP controller differs from the oracle controller, as the SOCP controller chooses control inputs that are robust to the prediction uncertainty. Consequently, the controller behaves more conservatively; in this case, it leads to a slightly faster walking gait than that of the oracle CLF-QP controller.

\vspace{-0.5em}
\subsection{Hardware Experiment: Cart-pole System}

The importance of the method presented in this work is most notable for real hardware systems, as we can use the data collected from the real system to account for the inevitable inaccuracies that even our best possible mathematical description of its dynamics might suffer from. This is precisely what is observed in the experiment we conducted on a Quanser Linear Servo Base Unit with Inverted Pendulum \cite{quanser_products_2021} hardware (Figure \ref{fig:systems} (b)). This cart-pole system consists of a linearly-actuated cart and an unactuated pendulum. The state of the system can be described as $x= [s, \dot{s}, \theta, \dot{\theta}] \in \R^4$, where $s$ and $\dot{s}$ are the cart's position and velocity, and $\theta$ and $\dot{\theta}$ correspond to the pole's relative angle with respect to the upright position and its angular velocity. The control input $u \in \R$ is the voltage applied to the linear actuator of the cart.

The control objective of this experiment is to swing-up the pole to the upright position and balance it at the top, while respecting a safety constraint on the cart's position, given as $|s| \le s_{\text{lim}}\!=\!0.35$\qty{}{\meter}. In particular, this constraint is placed to avoid the cart from colliding against the limits of the linear guide. The CBF we designed, which is then used as the certificate function, is based on the exponential CBF design methods for high relative-degree constraints \cite{nguyen2016exponential}; in our case, the original cart position constraint has a relative degree of two. This results in a CBF expressed as
\begin{equation}
    C(x) = -2 s \dot{s} + k (s^2_{\text{lim}} - s^2),
\end{equation}
whose zero-level set is depicted in red in Figure \ref{fig:cart_pole_results} left.

\begin{table}[t]
\caption{Total execution time (data selection, GP inference, numerical optimization) of the GP-CF-SOCP controller with different datasets in the RABBIT simulation and the Cart-pole experiment. Mean and standard deviations are in milliseconds.
}
\label{table:results}
\centering
\begin{tabular}{|c|crr|crr|}
\hline
\multicolumn{1}{|r|}{\textbf{}} & \multicolumn{3}{c|}{\textbf{GP-CF-SOCP (\ours)}}                                           & \multicolumn{3}{c|}{\textbf{GP-CF-SOCP (Full)}}                                             \\ \hline
System                          & \multicolumn{1}{c|}{mean} & \multicolumn{1}{c|}{stdev} & \multicolumn{1}{c|}{M} & \multicolumn{1}{c|}{mean} & \multicolumn{1}{c|}{stdev} & \multicolumn{1}{c|}{N} \\ \hline
RABBIT                          & \multicolumn{1}{r|}{22.9}       & \multicolumn{1}{r|}{4.4}        & 30                     & \multicolumn{1}{r|}{238.3}        & \multicolumn{1}{r|}{9.4}         & 12765                  \\ \hline
Cart-Pole                       & \multicolumn{1}{r|}{11.8}       & \multicolumn{1}{r|}{0.75}        & 40                     & \multicolumn{1}{r|}{60.4}      & \multicolumn{1}{r|}{4.1}        & 6957                   \\ \hline
\end{tabular}
\vspace{-1em}
\end{table}

For the swing-up task, we design a reference policy $\uref$, which is a hybrid controller that switches between an energy-based feedback controller and a stabilizing controller to which the system switches at the vicinity of the equilibrium.

\begin{figure}
\centering
\adjustbox{trim=20pt 20pt 60pt 20pt,clip,max width=\columnwidth}{
    \includegraphics{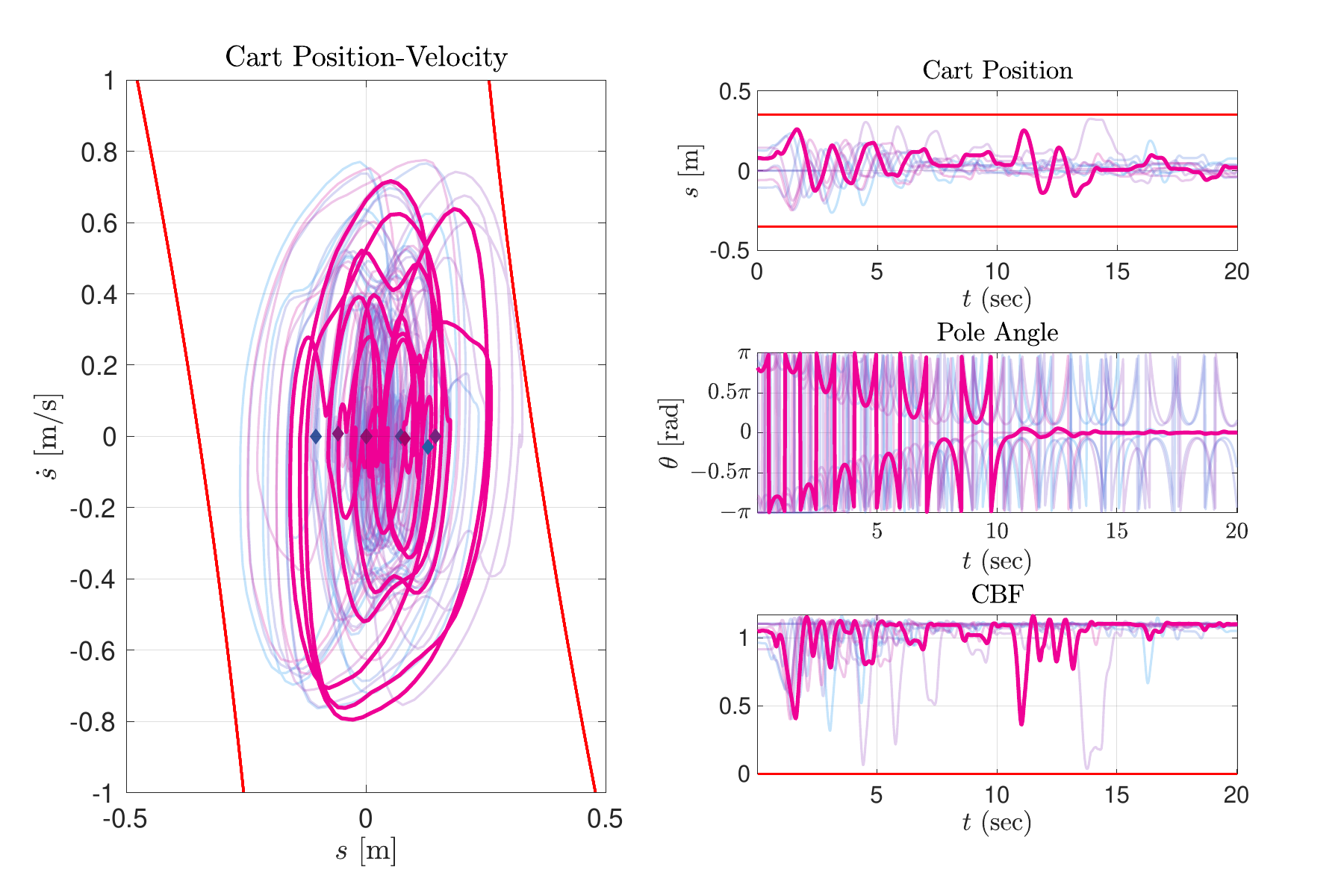}
}
\vspace{-1em}
\caption{10 episodes of the cart-pole experiment under GP-CBF-SOCP (Ours). We highlight one of the ten trajectories in \textmagenta{magenta} and the rest in transparent curves. 
On the left is the phase plot of the trajectories in the cart position and velocity space ($s\;\text{vs}\;\dot{s}$). The region between the red curves indicates the zero-super level set of the CBF. The diamond markers indicate the initial states of the trajectories. No trajectory exits the zero-super level set of the CBF. On the right are the plots of cart position($s$), pole angle ($\theta$), and CBF value of the trajectories in time. The highlighted trajectory successfully swings up the pole while maintaining the safety constraint.}
\label{fig:cart_pole_results}
\vspace{-1em}
\end{figure}

We apply the reference policy filtered by the nominal model-based CBF-QP certifying filter. We use the high-fidelity dynamics model provided by the manufacturer \cite{quanser_products_2021} as the nominal model. The computed filtered input is applied to the system every 25ms, which is the target sampling time for the real-time execution of the controller.
We observe that the nominal model-based CBF-QP fails to satisfy the safety constraints at several trials and the CBF becomes negative. 

This motivates us to employ the GP-CBF-SOCP certifying filter. In order to maintain feasibility of the GP-CBF-SOCP filter, the dataset must sufficiently cover the state and control input space where the system operates. We collect these data points in an episodic fashion. As more data points are aggregated, the GP inference takes longer, eventually exceeding the 25ms limit of our sampling time. Thus, we conduct the episodic procedure twice, each collecting nine trajectories, and then combine the two datasets into the full dataset. With the full dataset comprising 6957 ($N$) data points, we observe that the GP-CBF-SOCP controller takes too long to perform the inference, causing an average 60.4ms execution time (Table \ref{table:results}), which does not meet the target sampling rate requirement. This effect is evident in the experiment, as the cart-pole fails to swing up properly due to the delay.

On the contrary, using our data selection algorithm with 40 ($M$) points selected online, the total execution time becomes much smaller, resulting in an average of 11.8ms. Over 10 experiments using our method, we achieve $100 \%$ constraint satisfaction. These trajectories are shown in Figure \ref{fig:cart_pole_results}. Although not all of these experiments result in a successful balance at the upright position within the allocated 20 seconds (achieved in 6 out of 10 experiments), the GP-CBF-SOCP successfully prioritizes safety over performance, ensuring the cart never exits the defined limits imposed by the CBF-based certifying constraint. In the \href{https://youtu.be/MFgGRcuL--8}{video}\footnote{\href{https://youtu.be/MFgGRcuL--8}{https://youtu.be/MFgGRcuL-\!\;-8}} showcasing the results, it is clear that the learned certifying constraint forces the cart to drop the pole when it approaches the position limit. Moreover, we demonstrate that even when an external user introduces disturbances by pushing the pole, the system remains safe.

\vspace{-0.25em}
\section{Conclusion}
\label{sec:conclusion}
In this study, we introduce a runtime-efficient data-driven certifying filter approach applicable to real-time, complex robotic systems with uncertain dynamics that typically require large datasets for learning certified control laws. We achieved this by creating a nonparametric learning-based SOCP filter with significantly improved time complexity, transitioning from quadratic to linear with respect to dataset size, utilizing a novel online data selection algorithm. This algorithm generates a dataset most relevant to the desired certification of the system, \revision{grounded in a theoretical analysis that confirms its optimality with respect to the lower bound under reasonable dataset assumptions}. The effectiveness of our algorithm is demonstrated in securing the safety of a real-world cart-pole swing-up task and maintaining stable locomotion for a five-link bipedal walker under significant mass uncertainty, as exhibited in the RABBIT simulation.

Our investigation into the quantification of information from individual data points for the certifying filter establishes a foundation for a more profound understanding of the relationship between data and certifying control laws. 
Potential future research directions include incorporating data selection objectives during the data collection process for sample-efficient learning and expanding our framework to address systems with multiple certifying constraints. 

\vspace{-0.25em}
\section*{Acknowledgment}
We thank Andrew J. Taylor, Victor D.
Dorobantu, and Alonso Marco for insightful discussions. This work is supported by DARPA Assured Autonomy Grant No. FA8750-18-C-0101, the NASA ULI Grant No. 80NSSC20M0163, and the NSF
Grant CMMI-1944722. Any opinions, findings, and conclusions or recommendations expressed in this material do not necessarily reflect the views of any aforementioned organizations. \revision{We thank the reviewers for their constructive comments, including the simplified proof of Lemma \ref{lemma:flippingsigns}.}
\vspace{-0.25em}
\appendix
\vspace{-0.5em}
\subsection{Proof of Lemma \ref{lemma:suffcient_cond}}
\label{appendix:lemma1proof}
For notational convenience, we will drop $(\cdot|\dataset)$. From \eqref{eq:socp-cbf-constraint2}, GP-CF-SOCP is feasible under $u\!= \!\alpha' \soccsimple^\top$ if
\vspace{-0.5em}

{\small \begin{equation*}
    c \alpha \norm{\soccsimple}^2 - \beta \gpsigmaB \!\left(x, c\alpha \soccsimple^\top\right) \ge -\left(\widehat{\LfC}(x)+\gamma(\cert(x))\right),
\end{equation*}}

\noindent where $c:=\alpha'/\alpha > 1$.
First, we compare $\gpsigmaB \!\left(x, c\alpha \soccsimple^\top\right)$ and $\gpsigmaB \!\left(x, \alpha \soccsimple^\top\right)$ as below:
\vspace{-0.5em}

{\small \begin{align*}
    & \gpsigmaB^2 \!\left(x, c\alpha \soccsimple^\top\right) = [1 \;\; c\alpha \soccsimple] \gpGramB(x)\begin{bmatrix}1 \\ c\alpha \soccsimple^\top \end{bmatrix} \\
    & = c^2 [1/c \;\; \alpha \soccsimple] \gpGramB(x)\begin{bmatrix}1/c \\ \alpha \soccsimple^\top \end{bmatrix} \\
    & = c^2 \! \left( \! [1 \;\; \alpha \soccsimple] \gpGramB(x)\!\begin{bmatrix}1 \\ \alpha \soccsimple^\top \end{bmatrix} \!-\! \left[1\!\!-\!\!\frac{1}{c^2} \;\; 0\right] \!\gpGramB(x)\!\begin{bmatrix}1 - \frac{1}{c^2} \\ 0 \end{bmatrix} \right) \\
    & = c^2 \left( \gpsigmaB^2 \!\left(x, \alpha \soccsimple\right) - \left(1-\frac{1}{c^2} \right)^2 \gpGramB(x)_{[1, 1]} \right)
\end{align*}}

\noindent Using this expression, we can check that
\vspace{-0.5em}

{\footnotesize \begin{align*}
    & \frac{c \alpha \norm{\soccsimple}^2 - \beta \gpsigmaB \!\left(x, c\alpha \soccsimple^\top\right)}{c \left( \alpha \norm{\soccsimple}^2 - \beta \gpsigmaB \!\left(x, \alpha \soccsimple^\top\right) \right)} = \\
    & \frac{\alpha \norm{\soccsimple}^2 \!-\! \beta \sqrt{\gpsigmaB^2 \!\left(x, \alpha \soccsimple^\top\right) \!-\!\left(1-\frac{1}{c^2} \right)^2 \gpGramB(x)_{[1, 1]}}}{\alpha \norm{\soccsimple}^2 - \beta \gpsigmaB \!\left(x, \alpha \soccsimple^\top\right)} > 1.
\end{align*}}

\noindent Finally, since $\alpha \norm{\soccsimple}^2 \!-\! \beta \gpsigmaB \!\left(x, \alpha \soccsimple^\top\right)$ is strictly positive from \eqref{eq:suffcient_cond}, by taking $c$ satisfying
\begin{equation*}
    c \ge \frac{-\left(\widehat{\LfC}(x)+\gamma(\cert(x))\right)}{\alpha \norm{\soccsimple}^2 - \beta \gpsigmaB \!\left(x, \alpha \soccsimple^\top\right)},
\end{equation*}
we get
\begin{align*}
    & c \alpha \norm{\soccsimple}^2 - \beta \gpsigmaB \!\left(x, c\alpha \soccsimple^\top\right) \\
    & >c \left( \alpha \norm{\soccsimple}^2 - \beta \gpsigmaB \!\left(x, \alpha \soccsimple^\top\right) \right) \\
    & \ge -\left(\widehat{\LfC}(x)+\gamma(\cert(x))\right),
\end{align*}
which completes the proof.
\hfill \qedsymbol

\subsection{Proof of Lemma \ref{lemma:data_select_objective}}
\label{appendix:lemma2proof}
For notational convenience, we will drop $(\cdot|\datasetonline)$. We begin the proof by noting that
{\small \begin{align}
    & \lim_{\alpha\rightarrow \infty } \frac{1}{\alpha} \begin{bmatrix} \kadp_{*1}(x, \alpha\soccsimple^\top) \\ \vdots \\ \kadp_{*M}(x, \alpha\soccsimple^\top) \end{bmatrix} \nonumber \\
    & = \lim_{\alpha\rightarrow \infty } \frac{1}{\alpha} \begin{bmatrix} k_f(x_1, x_1) + \kadpu_{*1}(x, \alpha\soccsimple^\top) \\ \vdots\\k_f(x_M, x_M) + \kadpu_{*M}(x, \alpha\soccsimple^\top) \end{bmatrix} \nonumber \\
    & = \begin{bmatrix} \kadpu_{*1}(x, \soccsimple^\top) \\ \vdots\\ \kadpu_{*M}(x, \soccsimple^\top) \end{bmatrix}.
    \label{eq:prooflemma2eq}
\end{align}}
Thus,
{\small \begin{align*}
\arg & \min_{\datasetonline} \lim_{\alpha\rightarrow \infty }  \frac{1}{\alpha}\gpsigmaB\!\left(\!x, \alpha \soccsimple^\top\!\right) \!= \!\arg \min_{\datasetonline} \frac{1}{\alpha^2}\gpsigmaB^2 \!\left(\!x, \alpha \soccsimple^\top\!\right) \\
    = & \arg \max_{\datasetonline}\!\lim_{\alpha\rightarrow \infty }\!\frac{1}{\alpha^2} \!\left[\kadp_{*1}\!(x, \alpha\soccsimple^\top\!) \cdots \kadp_{*M}(x, \alpha\soccsimple^\top\!) \!\right] \\
    & \hspace{20pt} (\Kadponline + \sigma_n^2 I )^{-1} \begin{bmatrix} \kadp_{*1}(x, \alpha\soccsimple^\top) \\ \vdots \\ \kadp_{*M}(x, \alpha\soccsimple^\top) \end{bmatrix} \;\;(\text{from \eqref{eq:sigma_adp2}}) \\
    = & \arg \max_{\datasetonline} \left[ \kadpu_{*1}(x, \soccsimple^\top) \; \cdots \; \kadpu_{*M}(x, \soccsimple^\top) \right] \\
    & \hspace{25pt} (\Kadponline + \sigma_n^2 I )^{-1} \begin{bmatrix} \kadpu_{*1}(x, \soccsimple^\top) \\ \vdots\\ \kadpu_{*M}(x, \soccsimple^\top) \end{bmatrix} \;\;(\text{from \eqref{eq:prooflemma2eq}}),
\end{align*}}
\noindent which is precisely the objective function in the Lemma.
\hfill \qedsymbol

\subsection{Proof of Theorem \ref{thm:main}}
\label{appendix:thm1proof}

For notational convenience, we use the subscript $ij$ to indicate the $(i, j)$-th element of a matrix. We first present a few lemmas that will be used in the proof.

\begin{lemma} 
\label{lemma:perronfrobenius}
Let $C = \left(c_{ij}\right) \in \R^{n \times n}$ be a non-negative matrix. Then, the maximal eigenvalue of $C$ is upper bounded by its maximal row sum, that is,
\begin{equation}
\lambda_{\max}(C) \le \max_{i} \sum_{j=1}^{n} c_{ij}.
\end{equation}
\end{lemma}

\begin{proof}
This is a corollary of Perron-Frobenius Theorem for nonnegative matrices \cite[Ch.8]{meyer2000matrix}.
\end{proof}

\begin{lemma}[Weyl's Inequality] 
\label{lemma:weyl}
Let $A, B \in \R^{n \times n}$ be symmetric matrices. Then,
\begin{equation*}
    \lambda_{\min}(A + B) \ge \lambda_{\min}(A) + \lambda_{\min}(B).
\end{equation*}

\end{lemma}

\begin{lemma}
\label{lemma:flippingsigns}
Let $S = (s_{ij}) \in \R^{n \times n}$ have diagonal entities satisfying $s_{ii} = 1$, and off-diagonal entities satisfying \revision{$s_{ij} \le 0$} for all $i \neq j$. Let $\bar{S} = (\bar{s}_{ij})\in \R^{n \times n}$ be a matrix whose diagonal entities are all one, and whose off-diagonal entities are $\bar{s}_{ij} = \pm s_{ij}$, where the signs can be arbitrary. Then, if $S$ is positive definite, $\bar{S}$ is also positive definite.
\end{lemma}

\begin{proof} \revision{We denote $|\cdot|$ as the elementwise absolute value operation of a vector or a matrix. Define $T:=S-I$, and $\bar{T}:=\bar{S}-I$. Since $S$ is positive definite, for all unit vector $\hat{v} \in \R^n$ (s.t. $||\hat{v}||_2=1$), we have
\begin{equation*}
    \hat{v}^\top S \hat{v} = \hat{v}^\top (1 + T) \hat{v} = 1 + \hat{v}^\top T \hat{v} > 0,
\end{equation*}
Since $|T| = -T$ because $s_{ij} \le 0$ for $i \neq j$, $\hat{v}^\top |T| \hat{v} < 1$ for all $\hat{v} \in \R^n$ s.t. $||\hat{v}||_2=1$. It also holds that 
\begin{equation*}
    |\hat{v}|^\top |T| |\hat{v}| < 1,
\end{equation*}
since $\norm{|\hat{v}|}_2=1$. Then, for all unit vector $\hat{v} \in \R^n$, we have
\begin{align*}
\hat{v}^\top \bar{S} \hat{v} & = \hat{v}^\top (I + \bar{T}) \hat{v} = 1 + \hat{v}^\top \bar{T} \hat{v} = 1 + \sum_{i\neq j} \bar{s}_{ij} v_i v_j \\
& \ge 1 - \sum_{i\neq j} |s_{ij} v_i v_j| = 1 - |\hat{v}|^\top |T| |\hat{v}| > 0.
\end{align*} 
}
\end{proof}


Presented next is the main Proof of Theorem \ref{thm:main}. We will drop $(x, u)$ from $\kadpu_{*i}$ and $\nqi$, and $(\cdot|\datasetonline)$ for notational convenience.
We want to prove

{\small
\begin{align*}
\eqref{eq:main-lower-bound} & \Leftrightarrow
\; \left[ \kadpu_{*1} \; \cdots \; \kadpu_{*M} \right] (\Kadponline + \sigma_n^2 I )^{-1} \begin{bmatrix} \kadpu_{*1} \\ \vdots \\ \kadpu_{*M} \end{bmatrix} \\
& \ge \frac{1 - \epsilon}{1 + \epsilon(M-2)} \left[ \kadpu_{*1} \; \cdots \; \kadpu_{*M} \right] \Diag\!\left(\!\left[\!\frac{1}{\kadp_{1}} \cdots \frac{1}{\kadp_{M}}\!\right]\!\right)\!\begin{bmatrix} \kadpu_{*1} \\ \vdots \\ \kadpu_{*M} \end{bmatrix}. \nonumber
\end{align*}}
It is sufficient to prove that 
\begin{equation*}
    (\Kadponline + \sigma_n^2 I )^{-1} \succeq \frac{1 - \epsilon}{1 + \epsilon(M-2)} \Diag\left(\left[\frac{1}{\kadp_{1}} \cdots \frac{1}{\kadp_{M}}\right]\right),
\end{equation*}
and this is equivalent to
\begin{equation}
\label{eq:proof-step1}
    \frac{1 + \epsilon(M-2)}{1 - \epsilon} \Diag\left(\left[\kadp_{1} \;\cdots\; \kadp_{M}\right]\right) - (\Kadponline + \sigma_n^2 I) \succeq 0.
\end{equation}
We have

{\small \begin{align*}
    & \frac{1 + \epsilon(M-2)}{1 - \epsilon} \Diag\left(\left[\kadp_{1} \;\cdots\; \kadp_{M}\right]\right) - (\Kadponline + \sigma_n^2 I) \nonumber \\
    = & \begin{bmatrix} 
    \frac{\epsilon (M-1)}{1-\epsilon}\kadp_{1} - \sigma_n^2 & & - \kadp_{ij} \\
     & \ddots &  \\
    -\kadp_{ji} & & \frac{\epsilon (M-1)}{1-\epsilon}\kadp_{M} - \sigma_n^2
     \end{bmatrix} \nonumber \\
     = & \Diag\!\left(\!\sqrt{\kadp_{1}} ,\cdots,\sqrt{\kadp_{M}}\right)\! A \;\Diag\!\left(\!\sqrt{\kadp_{1}} \;\cdots\; \sqrt{\kadp_{M}}\right), \label{eq:proof1} 
\end{align*}}
where 
\begin{equation*}
A: = \begin{bmatrix} 
    \frac{\epsilon (M-1)}{1-\epsilon} - \frac{\sigma_n^2}{\kadp_1} \!& & \! - \frac{\kadp_{ij}}{\sqrt{\kadp_{i}\kadp_{j}}} \\
    & \ddots & \\
    -\frac{\kadp_{ji}}{\sqrt{\kadp_{j}\kadp_{i}}} \!&  &\! \frac{\epsilon (M-1)}{1-\epsilon} -\frac{\sigma_n^2}{\kadp_M}
     \end{bmatrix}.
\end{equation*}

\noindent Thus, it is sufficient to prove that $A$ is positive semidefinite. By Lemma \ref{lemma:weyl},
\begin{align*}
\lambda_{\min}(A) \ge \lambda_{\min}(\epsilon(M-1)\bar{S}) + \lambda_{\min}\left(A - \epsilon(M-1) \bar{S}\right),
\end{align*}
where
\begin{equation*}
    \bar{S}:= \begin{bmatrix} 
    1 \!& & \! - \frac{1}{\epsilon (M-1)} \frac{\kadp_{ij}}{\sqrt{\kadp_{i}\kadp_{j}}} \\
    & \ddots & \\
    -\frac{1}{\epsilon (M-1)}\frac{\kadp_{ji}}{\sqrt{\kadp_{j}\kadp_{i}}} \!&  &\! 1
     \end{bmatrix}.
\end{equation*}
Note that
{\footnotesize \begin{align*}
    A \!- \!\epsilon(M-1) \bar{S} \!=
    \!\Diag\left(\frac{\epsilon^2 (M - 1)}{1-\epsilon} - \frac{\sigma_n^2}{\kadp_1}, \cdots, \frac{\epsilon^2 (M - 1)}{1-\epsilon} - \frac{\sigma_n^2}{\kadp_M} \right),
\end{align*}}

\noindent and from \eqref{eq:noise_cond}, $\frac{\epsilon^2 (M - 1)}{1-\epsilon} - \frac{\sigma_n^2}{\kadp_i} \ge 0$ for all $i=1,\cdots, M$, thus, $A - \epsilon(M-1) \bar{S}$ is positive semidefinite. Therefore, it is now sufficient to prove that $\bar{S}$ is positive semidefinite.

We define
\begin{equation*}
    C:= \begin{bmatrix}
    0 & & \frac{1}{\epsilon(M-1)} \frac{|\kadp_{ij}|}{\sqrt{\kadp_i}\bar{k_{j}}} \\
    & \ddots & \\
    \frac{1}{\epsilon(M-1)} \frac{|\kadp_{ji}|}{\sqrt{\kadp_j}\bar{k_{i}}} & & 0
    \end{bmatrix}.
\end{equation*}
By applying Lemma \ref{lemma:perronfrobenius} to $C$ which is non-negative, and by using condition \eqref{eq:correlation}:
\vspace{-0.5em}

\small
\begin{equation*}
    \lambda_{\max}(C) \!\le\!\max_{i}\!\!\sum_{j=1, j\neq i}^{M} \frac{1}{\epsilon(M-1)} \frac{|\kadp_{ij}|}{\sqrt{\kadp_{i}\kadp_{j}}} < \frac{1}{\epsilon(M-1)} \; \epsilon (M-1)\!=\!1.
\end{equation*}
\normalsize

\vspace{-0.5em}
Thus, we have $\lambda_{\max}(C) < 1$. By Lemma \ref{lemma:weyl}, we have
\begin{equation*}
    \lambda_{\min}(I - C) \ge \lambda_{\min}(I) + \lambda_{\min}(-C) = 1 - \lambda_{\max}(C) > 0.
\end{equation*}
Thus, $S:=I - C$ is positive definite. Note that $\bar{S}$ and $S$ satisfy the conditions in Lemma \ref{lemma:flippingsigns} since $0 \le \frac{1}{\epsilon(M-1)}\frac{|\kadp_{ij}|}{\sqrt{\kadp_{i}\kadp_{j}}}$. Thus, by Lemma \ref{lemma:flippingsigns}, $\bar{S}$ is positive definite. 
\hfill
\qedsymbol


\subsection{Alternative lower bound of $\dataselectobjective(x, u)$}
\label{appendix:alternative-theorem}

\noindent 

\revision{
\begin{theorem}
\label{thm:main_alternative} We define \textit{modified kernel-based alignment measure} as
\begin{equation}
\nnqi(x, u) := \frac{|\kadpu_{*i}(x, u)|}{\sqrt{\kadp_i + \frac{1 - \epsilon}{1 + \epsilon(M-2)} \sigma_n^2}}. \label{eq:normalized_kernel_metric2}
\end{equation}
For a given dataset $\mathbb{D}_M$ with $M \ge 2$, assume that there exists a constant $\epsilon  \in [0, 1)$ that satisfies 
\begin{equation}
    \kadp_{ij}^2 < \epsilon^2 \kadp_{i} \kadp_{j},
\end{equation}
for all $i, j = 1, \cdots, M$ and $i\neq j$. Then, $\dataselectobjective(x, u)$ is lower bounded by the inequality below
\begin{equation}
\label{eq:main-lower-bound2}
\dataselectobjective(x, u) \ge \frac{1 - \epsilon}{1 + \epsilon(M-2)} \sum_{i=1}^{M} \nnqi^2(x, u).
\end{equation}
Note that the equality is satisfied when $\epsilon = 0$.
\end{theorem}

\noindent Compared to Theorem \ref{thm:main}, this lower bound does not require any condition on the noise variance $\sigma_n^2$. However, since $\nqi(x, u) \ge \nnqi(x, u)$, this is a looser bound compared to Theorem \ref{thm:main}.

\begin{proof}
The proof is very similar to the proof of Theorem \ref{thm:main} in Appendix \ref{appendix:thm1proof}. By following similar steps, it is sufficient to prove that

\begin{equation*}
    \frac{1 + \epsilon(M-2)}{1 - \epsilon} \Diag\left(\left[\kadp_{1} \;\cdots\; \kadp_{M}\right]\right) - \Kadponline \succeq 0.
\end{equation*}
The left hand side is equal to
\begin{equation*}
\Diag\!\left(\!\sqrt{\kadp_{1}} ,\cdots,\sqrt{\kadp_{M}}\right)\! A'\;\Diag\!\left(\!\sqrt{\kadp_{1}} \;\cdots\; \sqrt{\kadp_{M}}\right)    
\end{equation*}
where
\begin{equation*}
A'\!:=\!\begin{bmatrix} 
    \frac{\epsilon (M-1)}{1-\epsilon}\!& & \! - \frac{\kadp_{ij}}{\sqrt{\kadp_{i}\kadp_{j}}} \\
    & \ddots & \\
    -\frac{\kadp_{ji}}{\sqrt{\kadp_{j}\kadp_{i}}} \!&  &\! \frac{\epsilon (M-1)}{1-\epsilon}
     \end{bmatrix} \!= \epsilon(M-1) \bar{S} + \frac{\epsilon^2 (M - 1)}{1-\epsilon} I.
\end{equation*}
Since $\bar{S}$ is positive semidefinite from the proof of Theorem \ref{thm:main}, it is easily shown that $A'$ is positive semidefinite.
\end{proof}
} 

\subsection{Additional Experiment: Impact of Noise on Algorithm}
\label{appendix:noise-experiment}

\revision{
In this experiment, we examine how measurement noise impacts our data selection algorithm, particularly focusing on how prior knowledge of this noise can guide the choice of the maximum self-correlation threshold $\epsilon$. We use our running example in Section \ref{subsec:running_example_intro} to conduct this study.

We first define the metric to quantify the quality of the online dataset $\datasetonline$ given the state $x$. Recall that the objective of our data selection algorithm is to (indirectly) maximize the data selection objective \eqref{eq:optimization_subset_problem}.
Hence, we define the normalized true data selection objective as the quality metric, given as
\begin{align}
\label{eq:data_quality_metric}
I(\datasetonline|x) &= {\dataselectobjective(x, \socc) \over \kadp_{**}(x, \socc)},
\end{align}
which is the information ratio that $\datasetonline$ contains to characterize the target function $\Delta(x, \socc)$.
We further define the metric to evaluate the quality of the data selection algorithm with the threshold $\epsilon$ as
\begin{align}
\label{eq:data_quality_metric_algo}
\zeta(\epsilon) = \mathbb E_{x}\left[I\left(\datasetonline^{\epsilon}(x)|x\right)\right]
\end{align}
\noindent where $\datasetonline^{\epsilon}(x)$ is the online dataset constructed using Algorithm \ref{algo:sparsegp} with hyperparameter $\epsilon$ at the state $x$, and the expectation is taken over the state space empirically.

Figure \ref{fig:noise-experiment} illustrates the evaluated metric with varying noise levels $\sigma_{n, max}$ and the data selection hyperparameter $\epsilon$.
Artificial noise sampled from a uniform distribution $\mathcal{U}_{[-\sigma_{n, max}, \sigma_{n, max}]}$ is added to the dataset $\dataset$, and the GP is trained for each noise scale. 
The GP kernel hyperparameter $\sigma_{n}$ is set to $\sigma_{n, max}$. 
For each noise level, we construct the online dataset using the data selection algorithm with varying hyperparameters $\epsilon$. We evaluate the metric $\zeta(\epsilon)$ by uniformly sampling the state $x$ from the domain $[-1, 0] \times [0, 1]$ and taking the empirical mean. We employ $M=40$. As measurement noise increases, choosing a larger $\epsilon$ tends to maximize the metric. 
This suggests that allowing higher correlation between data points is beneficial in high-noise scenarios, but might be harmful in low-noise scenarios.
}

\begin{figure}
\centering
\includegraphics[width=0.9\columnwidth]{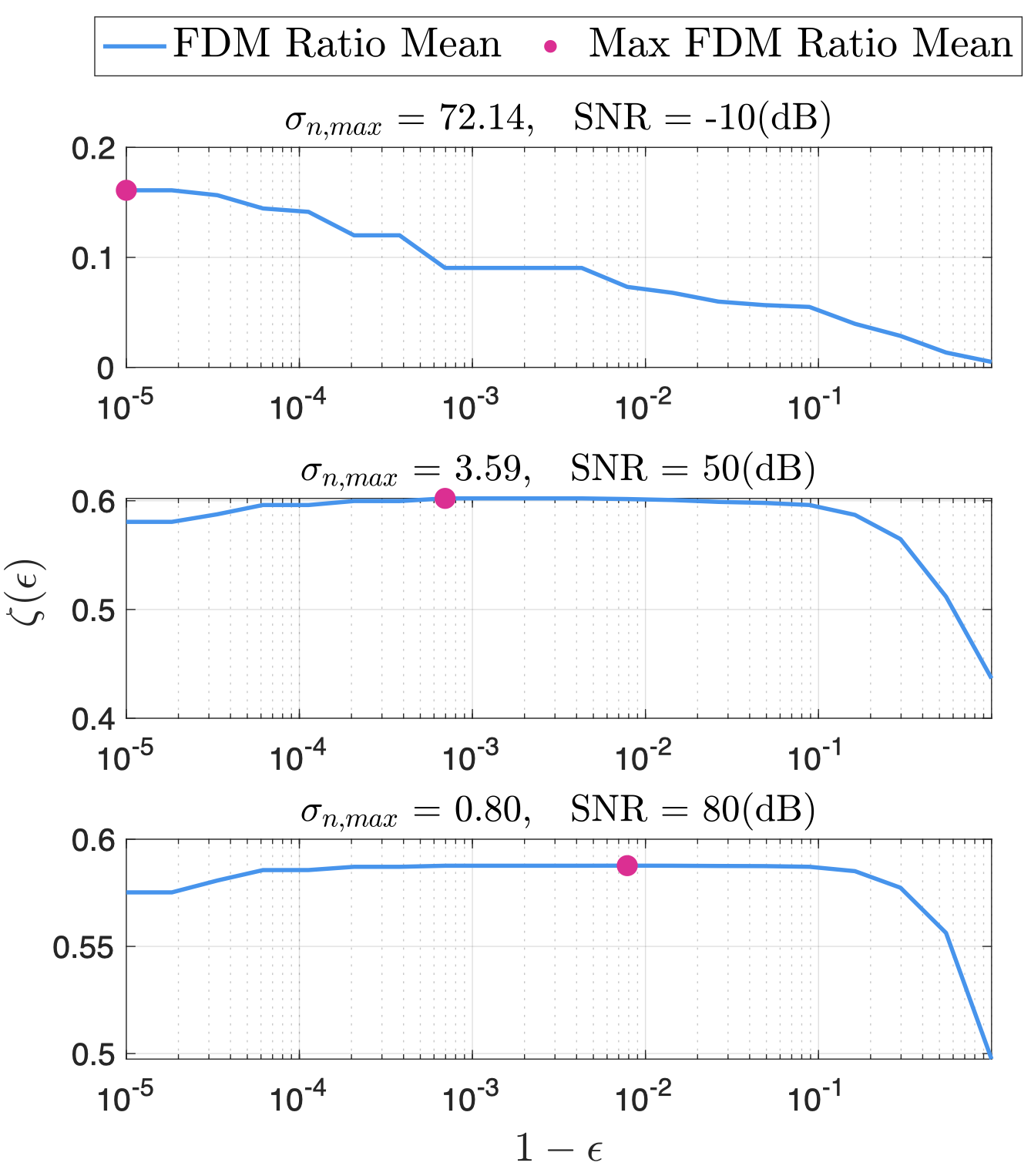}
\vspace{-0.5em}
\caption{
\revision{The trend of the data quality metric $\zeta(\epsilon)$, defined in \eqref{eq:data_quality_metric_algo}, in relation to the data selection hyperparameter $\epsilon$ at various levels of measurement noise. 
The hyperparameter $\epsilon$ represents the maximum allowable correlation in the constructed online dataset as per Algorithm \ref{algo:sparsegp}. 
Each row corresponds to a different noise scale $\sigma_{n, max}$, with the associated signal-to-noise ratio (SNR) indicated.
The x-axis of the plot is on a logarithmic scale.
The magenta dot indicates the optimal $\epsilon$ that attains the highest data quality metric for each noise scale.
}
}
\label{fig:noise-experiment}
\end{figure}

\bibliographystyle{IEEEtran}
\bibliography{reference.bib}{}

\begin{IEEEbiography}
[{\includegraphics[width=1in,height=1.25in,clip,keepaspectratio]{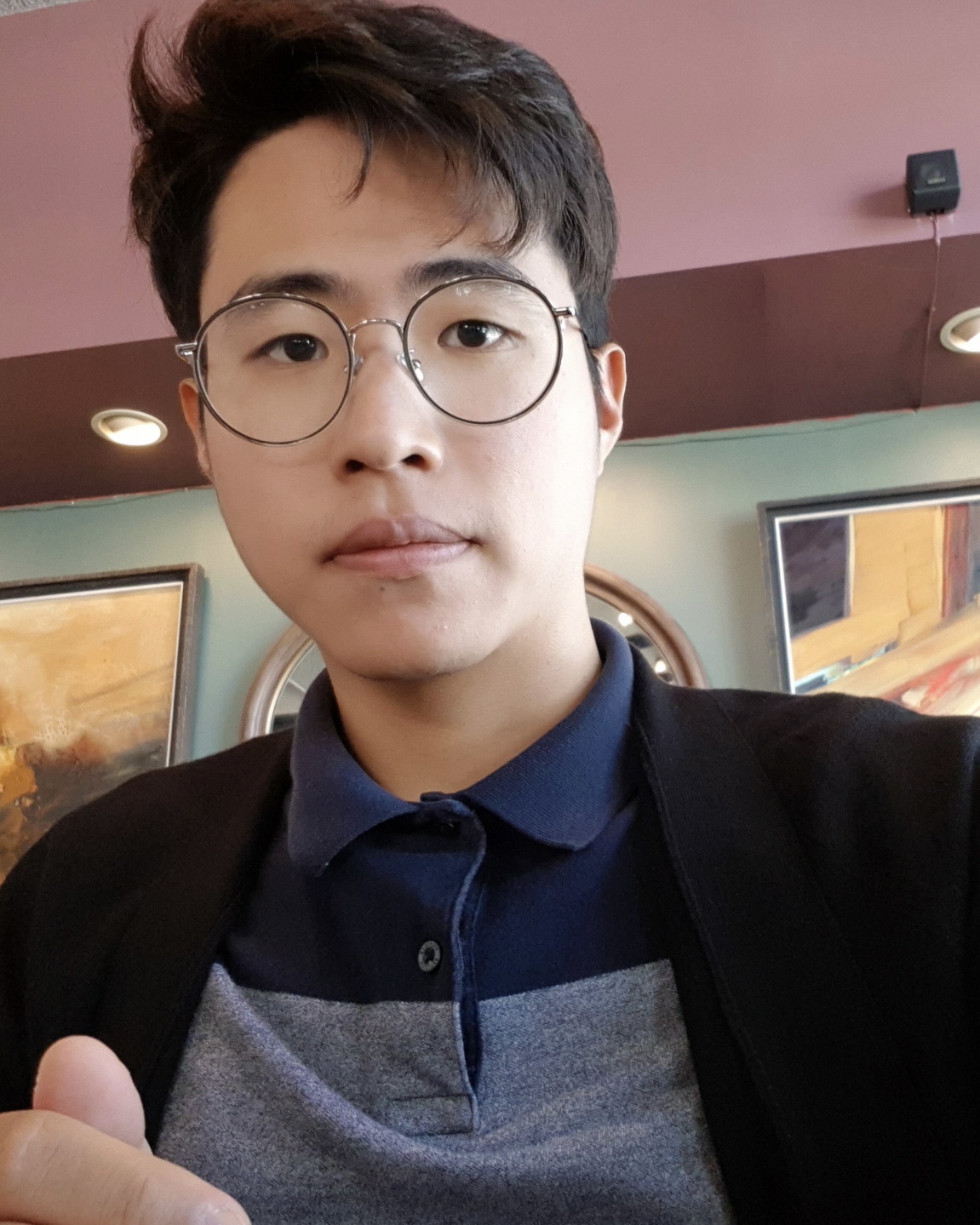}}]{Jason J. Choi} (Student Member, IEEE) received the B.S. degree in mechanical engineering from Seoul National University in 2019. He is currently pursuing a Ph.D. degree at University of California Berkeley in mechanical engineering. His research interests center on optimal control theories for nonlinear and hybrid systems, data-driven methods for safe control, and their applications to robotics and autonomous mobility.    
\end{IEEEbiography}

\begin{IEEEbiography}[{\includegraphics[width=1in,height=1.25in,clip,keepaspectratio]{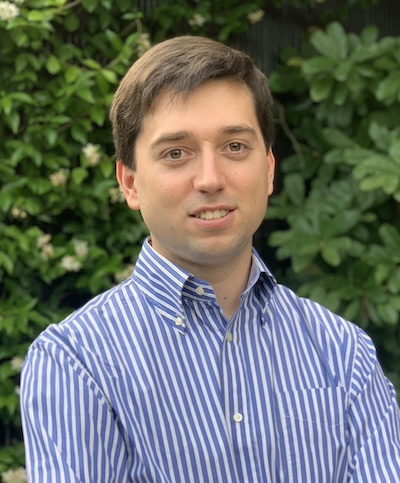}}]{Fernando Castañeda} received a Ph.D. degree in mechanical engineering from the University of California, Berkeley in 2023. He is a Rafael del Pino Foundation Fellow (2020) and a ``la Caixa" Foundation Fellow (2017). His research interests lie at the intersection of nonlinear control and data-driven methods, with a particular emphasis on ensuring the safe operation of high-dimensional systems in the real world. 
\end{IEEEbiography}

\begin{IEEEbiography}
[{\includegraphics[width=1in,height=1.25in,clip,keepaspectratio]{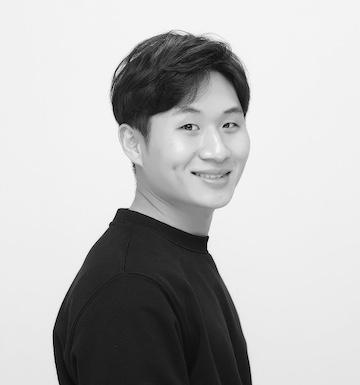}}]{Wonsuhk Jung} (Student Member, IEEE) received the B.S. degree in Mechanical Engineering and Artificial Intelligence from Seoul National University in 2022. He is currently pursuing a Ph.D. degree at Georgia Institute of Technology in Robotics. His research focuses on leveraging optimal control, planning, and data-driven methodologies for the safe operation of contact-rich robotics platforms.
\end{IEEEbiography}

\begin{IEEEbiography}
[{\includegraphics[width=1in,height=1.25in,clip,keepaspectratio]{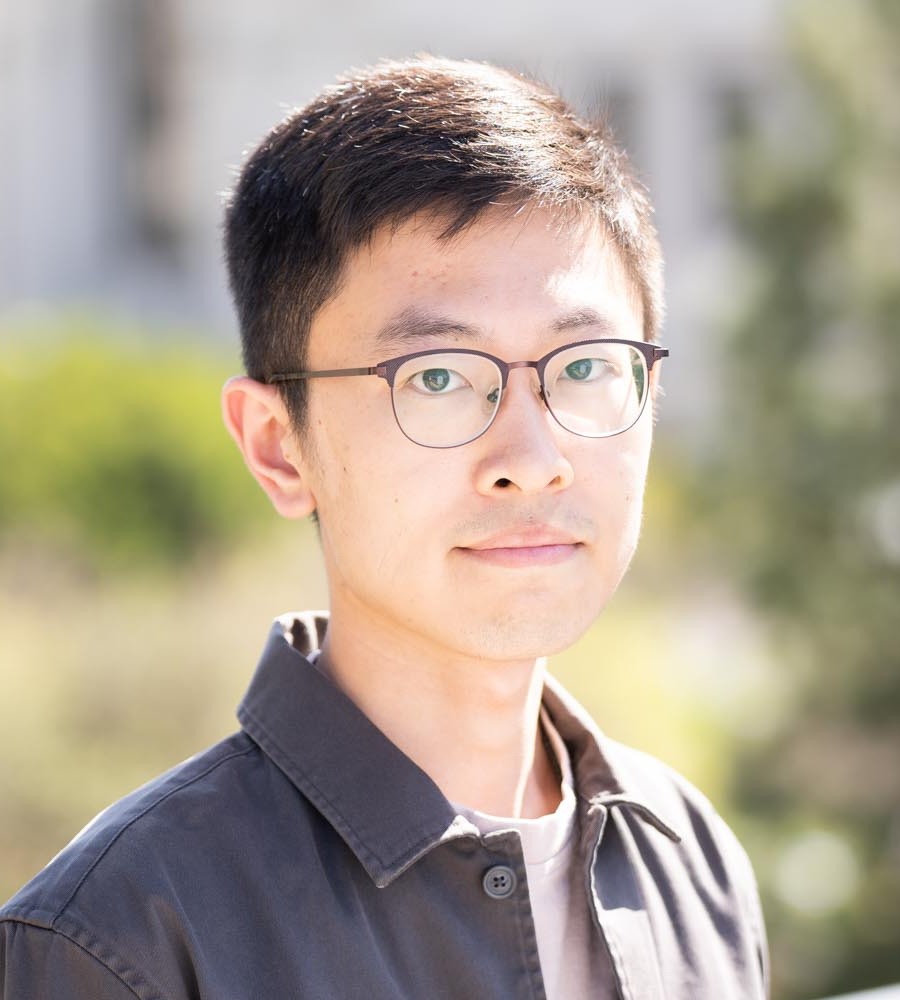}}]{Bike Zhang} (Student Member, IEEE) received the B.Eng. degree in electrical engineering and automation from Huazhong University of Science and Technology in 2017. He is currently working toward the Ph.D. degree in mechanical engineering at University of California Berkeley. His current research interests include predictive control and reinforcement learning with application to legged robotics.
\end{IEEEbiography}

\begin{IEEEbiography}
[{\includegraphics[width=1in,height=1.25in,clip,keepaspectratio]{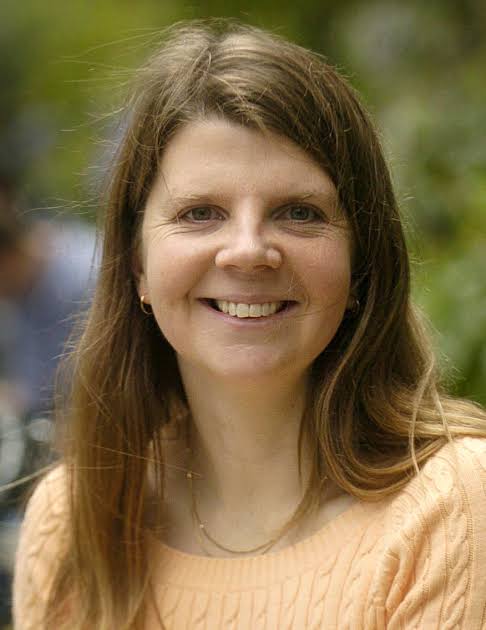}}]
{{C}laire J. Tomlin}{\,}(Fellow, IEEE) is the James and Katherine Lau Professor of Engineering and professor and chair of the Department of Electrical Engineering and Computer Sciences (EECS) at UC Berkeley. She was an assistant, associate, and full professor in aeronautics and astronautics at Stanford University from 1998 to 2007, and in 2005, she joined UC Berkeley. She works in the area of control theory and hybrid systems, with applications to air traffic management, UAV systems, energy, robotics, and systems biology. She is a MacArthur Foundation Fellow (2006), an IEEE Fellow (2010), and in 2017, she was awarded the IEEE Transportation Technologies Award. In 2019, Claire was elected to the National Academy of Engineering and the American Academy of Arts and Sciences.
\end{IEEEbiography}

\begin{IEEEbiography}
[{\includegraphics[width=1in,height=1.25in,clip,keepaspectratio]{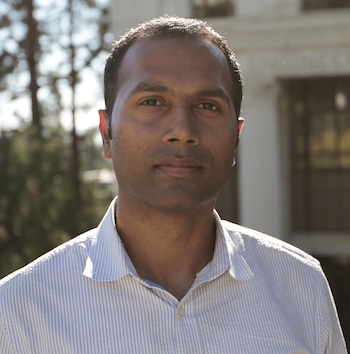}}]
{{K}oushil Sreenath}{\,}(Member, IEEE) is an associate professor of mechanical engineering, at UC Berkeley. He received a Ph.D. degree in electrical engineering and computer science and a M.S. degree in applied mathematics from the University of Michigan at Ann Arbor, MI, in 2011. He was a postdoctoral scholar at the GRASP Lab at University of Pennsylvania from 2011 to 2013 and an assistant professor at Carnegie Mellon University from 2013 to 2017. His research interest lies at the intersection of highly dynamic robotics and applied nonlinear control. He received the NSF CAREER, Hellman Fellow, Best Paper Award at the Robotics: Science and Systems (RSS), and the Google Faculty Research Award in Robotics.
\end{IEEEbiography}
\end{document}